\documentclass{article}

\usepackage[accepted]{icml2026}
\usepackage{url}

\usepackage[utf8]{inputenc} 
\usepackage[T1]{fontenc}    
\usepackage{url}            
\usepackage{booktabs}       
\usepackage{algorithmicx} 
\usepackage{amsfonts}       
\usepackage{nicefrac}       
\usepackage{microtype}      
\usepackage{xcolor}         
\usepackage{amsmath}
\usepackage{graphicx}
\usepackage[colorlinks=true, allcolors=blue]{hyperref}
\usepackage{algorithmicx}
\usepackage{multicol}
\usepackage{url}
\usepackage{multirow}
\usepackage{lscape}
\usepackage{wrapfig}
\usepackage{multicol}
\usepackage{subfigure}
\usepackage{subcaption}
\usepackage{amsmath}
\usepackage{amsthm}
\usepackage{algpseudocode}
\newtheorem{lemma}{Lemma}
\newtheorem{example}{Example}

\newtheorem{remark}{Remark}
\usepackage{pgfplots}
\usepackage{natbib}
\usepackage{caption}
\usepackage{amsmath}
\usepackage{amssymb}
\usepackage{mathrsfs}
\usepackage{minted}
\usepackage{listings}
\usepackage{tcolorbox}
\usepackage{polynom}
\usepackage{float}
\usepgfplotslibrary{external}
\usepackage{multirow}
\definecolor{lavender}{rgb}{0.9, 0.9, 0.98}
\usepackage{colortbl}

\usepackage[capitalize,noabbrev]{cleveref}

\usepackage[textsize=tiny]{todonotes}

\icmltitlerunning{FedReLa: Imbalanced Federated Learning via Re-Labeling}

\begin{document}

\twocolumn[
  \icmltitle{FedReLa: Imbalanced Federated Learning via Re-Labeling}



  \icmlsetsymbol{equal}{*}

  \begin{icmlauthorlist}
    \icmlauthor{Guangzheng Hu}{sch}
    \icmlauthor{Patricia Men\'endez}{sch}
    \icmlauthor{Feng Liu}{yyy}
    \icmlauthor{Mingming Gong}{sch}
    \icmlauthor{Guanghui Wang}{xxx}
    \icmlauthor{Liuhua Peng}{sch}
  \end{icmlauthorlist}

  \icmlaffiliation{yyy}{School of Computing and Information Systems, University of Melbourne, Victoria, Australia}
  \icmlaffiliation{xxx}{School of Statistics and Data Science, LPMC, KLMDASR, and LEBPS, Nankai University}
  \icmlaffiliation{sch}{School of Mathematics and Statistics, University of Melbourne, Victoria, Australia}

\icmlcorrespondingauthor{Liuhua Peng}{liuhua.peng@unimelb.edu.au}

  \icmlkeywords{Supervised Learning, Imbalanced Learning, Long-tailed Learning, Federated Learning}

  \vskip 0.3in
]



\printAffiliationsAndNotice{}  

\begin{abstract}
  Federated learning has emerged as the foremost approach for decentralized model training with privacy preservation. 
  The global class imbalance and cross-client data heterogeneity naturally coexist, and the mismatch between local and global imbalances exacerbates the performance degradation of the aggregated model. 
  The agnosticism of global class distribution poses significant challenges for data-level methods, especially under extreme conditions with severe class absence across clients.
  In this paper, we propose FedReLa, a novel data-level approach that tackles the coexistence of data heterogeneity and class imbalance in federated learning. By re-labeling samples with a feature-dependent label re-allocator, FedReLa corrects biased global decision boundaries without requiring knowledge of the global class distribution. This modular, model-agnostic approach can be integrated with algorithmic methods to deliver consistent improvements without additional communication overhead. Through extensive experiments, our method significantly improves the accuracy of minority classes and the overall accuracy on stepwise-imbalanced and long-tailed datasets, outperforming the previous state of the art.
\end{abstract}
\section{Introduction}

Federated learning (FL) facilitates collaborative model training across distributed clients without exchanging raw data, thereby preserving data privacy. Each client trains a model locally on private data and only shares parameter updates to a global server. Due to variations in client environments, such as differences in participation capacity, and geographic or demographic factors, data often exhibit significant heterogeneity over clients, leading to disparate model updates and suboptimal global model \citep{fedearly}. 

Imbalanced or long-tailed distributed data is common in real-world applications \citep{behavior,disease,fraud}
and even more prevalent in FL. Due to client-level data heterogeneity, two types of imbalance often coexist: local imbalance (within individual clients) and global imbalance (across the entire federation), both of which pose challenges for FL classification. 

Although both data- and algorithm-level methods are important for addressing class imbalance, research on imbalanced federated learning has largely focused on algorithm-level methods. Data-level approaches remain underdeveloped due to strict privacy constraints. Early algorithm-level work primarily tackles local imbalance via improved aggregation \citep{fedavg,fednova}, robust local training \citep{feddyn, scaffold, moon, fedprox}, selective client participation \citep{fedselect, fedselect1}, or architectural enhancements \citep{astraea}, while assuming all the classes are equally represented globally, a condition rarely met in practice.
Recent algorithmic works focus on more realistic scenarios where global class imbalance (e.g., step-wise or long-tailed distributions) \textit{coexists} with data heterogeneity. Ratio-loss \citep{ratioloss} and CLIMB \citep{climb} pioneered solutions for global step-wise imbalance and non-IID (not independent and identically distributed) client data. Subsequent studies \citep{fedrod, Fed-ETF, CReFF, FedLoGe, Fed-GraB} further tackled federated long-tailed (Fed-LT) learning.

Existing data-level methods (e.g., SMOTE \citep{smote123}) rely on global class prior information to identify minority classes before synthesizing new samples. However, in FL, local class priors are private, making access to global information or exchange of local class distributions impossible. Even in scenarios such as fraud detection \citep{fraud} or rare disease diagnosis \citep{cancer}, where global minority classes can be naturally identified, insufficient or even absent local minority samples in clients render synthesis unfeasible, leaving data-level methods underdeveloped.

To fill the gap in data-level FL methods, this paper proposes a novel data-level method that \textbf{eliminates the reliance on global class prior knowledge}, incurs \textbf{no additional communication overhead}, and involves \textbf{no training overhead} caused by extra trainable parameters, so as to improve the FL performance in challenging scenarios where both data heterogeneity and global class imbalance coexist. 

Unlike traditional data-level methods, such as SMOTE-based \citep{smote123, ADASYN} or mixup-based approaches \citep{ReMix, Selmix123}, our method operates purely in the label space, without synthesizing new features. Our novel method re-labels local data through a carefully designed feature-dependent label re-allocator that leverages minority-class information embedded in the global model.
Specifically, it identifies local majority class samples that intrude into the global minority-class feature space and implements asymmetric selective relabeling, implicitly enlarging the minority-class decision boundary. As a data-level approach for heterogeneous and class-imbalanced data in \textbf{Fed}erated Learning via \textbf{Re}-\textbf{La}beling, FedReLa has three key characteristics:

(i) \textbf{Model-Data Agnosticism}: Unlike methods that rely on balanced auxiliary data or explicit class priors \citep{fraud, ratioloss}, FedReLa is agnostic to model architecture, data format, and class distribution, and inherently improves data quality without domain-specific constraints.

(ii) \textbf{Lightweight Plug-in Adaptation}: Unlike prior methods \citep{CReFF, climb, FedLoGe} that require extra communicational or training cost from optimizing/uploading newly introduced parameters, FedReLa re-purposes the global model as a label re-allocator without introducing new trainable parameters, requiring no extra training or communicational burden (see Appendix \ref{apdx:cost}).

(iii) \textbf{Universal Composability}:  Operating solely in the label space, FedReLa integrates seamlessly with algorithm-level approaches and delivers consistent performance gains.

We evaluate FedReLa on Fashion-MNIST, CIFAR-10, CIFAR-100, and ImageNet under step-wise and long-tailed class imbalance, across varying degrees of data heterogeneity.
FedReLa consistently improves existing algorithm-level methods, achieving state-of-the-art performance.
In the most extreme cases, FedReLa boosts minority/tail-class accuracy by up to 38.30\% (step-wise) and 30.17\% (long-tailed) while maintaining overall accuracy superiority (Tables \ref{tab:step} and \ref{tab:lt}). These results conclusively demonstrate the superiority and applicability of FedReLa in FL. The code is available at: \url{https://github.com/guangzhengh/FedReLa.git}.

\section{Related works}
\paragraph{Imbalanced Learning.}
Imbalanced learning has been extensively studied in centralized settings, where data-level methods augment minority classes by generating synthetic samples using generative models \citep{ACGAN,BAGAN}, SMOTE-based techniques \citep{smote123,borderline123,ADASYN}, or feature interpolation via mixup \citep{ReMix,Selmix123,mixup}. However, their effectiveness is severely limited in FL due to scarcity of minority samples and absence of local classes, as well as the unavailability of global class priors, which prevents reliable identification of global minority classes and adjustment of augmentation strength for prior data-level methods.


\paragraph{FL for data heterogeneity.}
FL performance is primarily influenced by model aggregation, client updates, and local data distributions. 
Numerous FL methods primarily focus on client update and model aggregation to mitigate the adverse effects of skewed local datasets. During model aggregation, FedAvg \citep{fedavg} and FedNova \citep{fednova} pioneered weighted averaging local models based on local dataset sizes or batch sizes. For client update, regularization terms are incorporated into loss functions to penalize discrepancies between global and local models \citep{moon,fedprox} or constrain inter-round model divergence \citep{feddyn}. SCAFFOLD \citep{scaffold} introduced control variates to correct biased local gradients. However, these methods perform suboptimally on global minority/tail classes, as they mainly tackle local imbalance from data heterogeneity but neglect global class imbalance.

\paragraph{Data heterogeneity with global class imbalance.}
Several methods extend FL to settings with global imbalance, primarily via loss reweighting. Ratio-Loss \citep{ratioloss} estimates global class priors using auxiliary balanced data, while CLIMB \citep{climb} removes this requirement by introducing learnable loss weights, at the cost of increased training and communication overhead. 
Recent studies extend the scope from step-wise imbalance to long-tailed
distributions. CReFF \citep{CReFF} boosts tail-class performance via aggregated class feature retraining but incurs extra training and doubled communication costs, while FedROD \citep{fedrod} decouples global and local objectives yet may yield suboptimal global models when aggressively balancing local losses. FedETF \citep{Fed-ETF} enforces balanced feature learning via a fixed ETF (Equiangular Tight Frame) classifier head.

Based on the observation that head classes tend to have larger weight norms, FedGraB \citep{Fed-GraB} rescales the gradients of local models by class weight norms to enhance tail-class performance. As a follow-up, FedLOGE \citep{FedLoGe} further integrates the idea of \citet{Fed-ETF} by rescaling the weights of the fixed ETF classifier using the weight norms of an auxiliary classifier head. Despite this, our empirical findings in Section \ref{sec:experiments} reveal that weight norms become unreliable under high heterogeneity.

\paragraph{Other label-related concepts.} FedReLa is related to, but distinct from, learning with noisy labels~\cite{noisedeep,noisetolar2,noisesurvey} and label propagation~\cite{labelpro,labelpro1}.
Noisy-label methods typically aim to suppress harmful label corruption so that training is not misled, whereas FedReLa deliberately applies a controlled, minimal relabeling step to reshape biased decision boundaries induced by federated long-tailed and non-IID data.
Label propagation, in contrast, transfers information from labeled to unlabeled instances via graph-based similarity, often requiring global graph construction; FedReLa instead relabels local \emph{labeled} majority examples into minority classes and operates without building a global propagation graph.
FedReLa is a federated boundary-correction strategy tailored to class imbalance.
We therefore position FedReLa alongside these lines of work while emphasizing its novelty as intentional, locally executed relabeling for federated imbalance rather than noise removal or graph-based label diffusion.

\textbf{Motivations:} Existing methods tackle data imbalance through algorithmic adjustments. Why not improve local data quality directly? The reason is apparent: conventional data augmentation relies on global data distribution knowledge, which violates FL privacy constraints. The most relevant work, FedMix \citep{FedMix}, addresses heterogeneity using mixup. Still, it requires clients to share local data averages, which may require additional privacy-preserving mechanisms and increase communication costs. To our knowledge, no data-level FL approach mitigates the coexistence of data heterogeneity and global imbalance while achieving: (1) not requiring auxiliary datasets, (2) having a negligible additional computation cost with zero communication cost and no extra parameter training burden, and finally (3) being agnostic to global data distribution. 
This motivates FedReLa, a data-level method that simultaneously achieves all of the above requirements while significantly improving performance under extreme conditions.


\section{Federated Learning via Re-Labeling}\label{sec:theory}

In this section, we first analyze how local and global imbalances affect decision boundaries and why heterogeneity in globally imbalanced data exacerbates the performance impact of imbalanced data on global models. We then introduce the label re-allocator and analyze how re-labeled samples implicitly rebalance the biased global decision boundaries.

\subsection{Problem Formulation}

Consider a dataset $\mathcal{D}$ that contains data pairs $(X,Y) \sim P(x,y)$, where $x \in \mathcal{X} \subseteq \mathbb{R}^d$, $y \in \mathcal{Y} = \{1, 2, \ldots, C\}$  and $P$ represents the joint distribution.  
Denote the conditional distribution $X \mid Y = j \sim P_j(x)$ and the prior probability $\Pr(Y=j)=\pi_j$ for class $j \in \mathcal{Y}$. The marginal distribution of $X$ is then $P_X(x) = \sum_{j\in\mathcal{Y}} \pi_jP_j(x)$. Assume $\mathcal{D}$ is imbalanced 
with global imbalance ratio $\mathrm{IR}(\mathcal{D}) = \max_{j\in\mathcal{Y}}\pi_j/\min_{j\in\mathcal{Y}}\pi_j \gg 1$. 
Let $\eta_j(x) = \Pr(Y=j\mid X=x) = \pi_jP_j(x)/P_X(x)$ be the global posterior probability. Recalling that the Bayesian decision theorem \citep{bdt} defines the optimal estimated $y^*$ of a sample $x$ as $y^* = \mathrm{argmax}_{j\in\mathcal{Y}} \eta_j(x)$, the following result holds.

\begin{lemma}\label{lem:01}
    The optimal Bayesian decision boundary between two classes $j\neq \ell \in\mathcal{Y}$ is $S_{j,\ell} 
        \notag =  \left\{ x\in \mathcal{X}: \eta_j(x) = \eta_{\ell}(x) > \eta_{\ell^{\prime}}(x) ~~ \forall ~\ell^{\prime} \in \mathcal{Y} \setminus \{j,\ell\} \right\}$.
\end{lemma}

For $x\in S_{j,\ell}$, $\eta_j(x) = \eta_{\ell}(x)$ implies {$P_j(x)/P_{\ell}(x)= \pi_{\ell}/\pi_j$. 
Then for minority class $j$ and majority class $\ell$ with $\pi_j \ll \pi_{\ell}$, $S_{j,\ell}$ intrudes deeply into the minority class region, increasing the risk of misclassifying minority class samples.
This motivates balancing the ratio $\pi_{\ell}/\pi_j$ to \textit{shift the decision boundary back towards the majority class region}, thereby alleviating the adverse effects of class imbalance. 

In FL, the dataset $\mathcal{D}$ 
is distributed on $K$ clients with local datasets $\{\mathcal{D}^{(k)}\}_{k=1}^{K}$ and
assumes the class conditional distributions $\{P_{j}^{(k)}(x)\}_{k=1}^{K}$ are identical across all clients for each class $j$. 
In contrast, the class priors $\{\pi_{j}^{(k)}\}_{k=1}^{K}$ may be different among clients due to data heterogeneity. For classes $j$ and $\ell$, we have
$P_{j}^{(1)}(x)/P_{\ell}^{(1)}(x) = \cdots = P_{j}^{(K)}(x)/P_{\ell}^{(K)}(x)=P_j(x)/P_{\ell}(x)$. 
However, divergent class priors 
result in different local posterior probability $\eta_{j}^{(k)}(x)=\pi_{j}^{(k)}P_{j}(x)/\sum_{j_0\in\mathcal{Y}}\pi_{j_0}^{(k)}P_{j_0}(x)$ and misaligned Bayesian decision boundaries among clients.
This misalignment affects performance of aggregated classifier and slows down convergence of FL algorithms \citep{fedearly}.

Ideally, with properly chosen aggregation weights $\{w_k\}_{k=1}^{K}$, the decision boundary between classes $j$ and $\ell$ of the global aggregated model $\eta_{j}^{[w]}(x)=\sum_{k=1}^{K}w_k\eta_{j}^{(k)}(x)$ given by 
$$S_{j,\ell}^{[w]} =  \{x\in \mathcal{X}: P_j(x)/P_{\ell}(x)=\pi_{\ell}^{[w]}/\pi_j^{[w]}\}, $$ 
which can match $S_{j,\ell}$ in Lemma \ref{lem:01} by making $\pi_j^{[w]} = \pi_j$ and $\pi_{\ell}^{[w]} = \pi_{\ell}$, where $\pi_j^{[w]} = \sum^K_{k=1} w_k\pi_{j}^{(k)}/\sum^K_{k=1} w_k$ and $\pi_{\ell}^{[w]} = \sum^K_{k=1} w_k\pi_{\ell}^{(k)}/\sum^K_{k=1} w_k$.
For instance, setting $w_k = |\mathcal{D}^{(k)}|/|\mathcal{D}|$ achieves this alignment.
However, the global imbalance ratio $\pi_{\ell}/\pi_j$ still introduces bias into the aggregated decision boundary of the global model.
To address that, several algorithms \citep{logit,lossreweight1} have been proposed to adjust the ratio $\pi_{\ell}^{[w]}/\pi_j^{[w]}$ via alternative weighting schemes. 
Moreover, data heterogeneity can cause mismatches between global and local imbalance ratios, further complicating the class imbalance issue and amplifying bias in the aggregated decision boundary. 
{\color{black} See Example~\ref{example:01} in the Appendix for an illustration.}

\begin{remark}
    We adopt a global model based on posterior aggregation rather than parameter aggregation for two reasons. First, posterior aggregation renders changes in the decision boundary explicit and easier to quantify, which is essential for analyzing the effect of re-labeling in the next section. Second, it aligns with principles of statistical model averaging, offering a flexible and natural framework for heterogeneous FL.
    Furthermore, the two aggregation paradigms are approximately equivalent under mild regularity conditions,
    see Appendix \ref{appendix:A} for detailed discussions.
\end{remark}

\subsection{Aggregated decision boundary with re-labeled data} \label{sec:2.2}

We introduce a novel data-level approach that reallocates data labels to adjust the decision boundary by balancing class prior ratios at both local and global levels. This strategy also alleviates the mismatch between global and local imbalance ratios, and improves the robustness of the FL model.
Our proposed FedReLa is motivated by how re-labeling shifts decision boundaries locally and globally.
We first analyze its effect on a local client $k$, and then extend the discussion to model aggregation.
Without loss of generality, we consider a binary classification setting where classes $j$ and $\ell$ represent the minority and majority classes.

Let $(X^{(k)},Y^{(k)}) \sim P^{(k)}(x,y)$ denote the data pair for client $k$ with re-labeled $\widetilde{Y}^{(k)}$ and consider $\widetilde{\mathcal{D}}^{(k)}$ the corresponding re-labeled dataset.
Denote the probabilities of re-labeling $\ell$ to $j$ as $\rho_{\ell\to j}^{(k)}(x) = \Pr(\widetilde{Y}^{(k)} = j \mid X^{(k)} = x, Y^{(k)} = \ell)$ and re-labeling $j$ to $\ell$ as $\rho_{j\to\ell}^{(k)}(x) = \Pr(\widetilde{Y}^{(k)} = \ell \mid X^{(k)} = x, Y^{(k)} = j)$ for client $k$, then 
\begin{align}
    \notag \widetilde{\eta}_{j}^{(k)}(x) & = \Pr(\widetilde{Y}^{(k)} = j \mid X^{(k)} = x)\\
    &\notag = \eta_{j}^{(k)}(x)[1 - \rho_{j\to\ell}^{(k)}(x)] + [1-\eta_{j}^{(k)}(x)]\rho_{\ell\to j}^{(k)}(x).
\end{align}

\begin{lemma}\label{lem:bound_noise_k}
    Assume $\rho_{\ell\to j}^{(k)}(x^*) \leq 0.5$ and $\rho_{j\to\ell}^{(k)}(x^*) \leq 0.5$ for any $x^*\in \widetilde{S}^{(k)}$, then the optimal Bayesian decision boundary based on $\widetilde{\mathcal{D}}^{(k)}$ for client $k$ is \\ $\widetilde{S}^{(k)} = \left\{x^{*} \in \mathcal{X}: \frac{P_{j}(x^*)}{P_{\ell}(x^*)} = \frac{1-2\rho_{\ell\to j}^{(k)}(x^*)}{1-2\rho_{j\to\ell}^{(k)}(x^*)}\cdot\frac{\pi_{\ell}^{(k)}}{\pi_{j}^{(k)}}\right\}$.
\end{lemma}


When $\pi_{\ell}^{(k)}/\pi_{j}^{(k)} \gg 1$, we seek to achieve $[1-2\rho_{\ell\to j}^{(k)}(x^*)]/[1-2\rho_{j\to\ell}^{(k)}(x^*)]<1$ to locally adjust the decision boundary.
Given the scarcity of minority class samples in $\mathcal{D}^{(k)}$, it is reasonable to restrict re-labeling to occur only from majority class $\ell$ to minority class $j$, and set $\rho_{j\to\ell}^{(k)}(x)=0$.
Furthermore, since deeply invaded majority class samples are especially harmful, we design a label re-allocator where the re-labeling probability is proportional to the degree of intrusion. Specifically, with
$\rho_{\ell\to j}^{(k)}(x)\propto \eta_{j}^{(k)}(x)$, the Bayesian decision boundary is \\
$\widetilde{S}^{(k)} = \left\{x^{*} \in \mathcal{X}: \frac{P_j(x^*)}{P_{\ell}(x^*)} =[1-2\rho_{\ell\to j}^{(k)}(x^*)] \frac{\pi_{\ell}^{(k)}}{\pi_{j}^{(k)}}\right\}$.

Since $1-2\rho_{\ell\to j}^{(k)}(x^*) < 1$, $\widetilde{S}^{(k)}$ on re-labeled data shifts back to the majority
class region.
Based on re-labeled data $\widetilde{\mathcal{D}}=\cup_{k=1}^{K}\widetilde{\mathcal{D}}^{(k)}$, we also study the decision boundary of the global aggregated model $\widetilde{\eta}_{j}^{[w]}(x)=\sum_{k=1}^{K}w_k\widetilde{\eta}_{j}^{(k)}(x)$.

\begin{figure*}[h]
\centering
\includegraphics[width=0.76\textwidth]{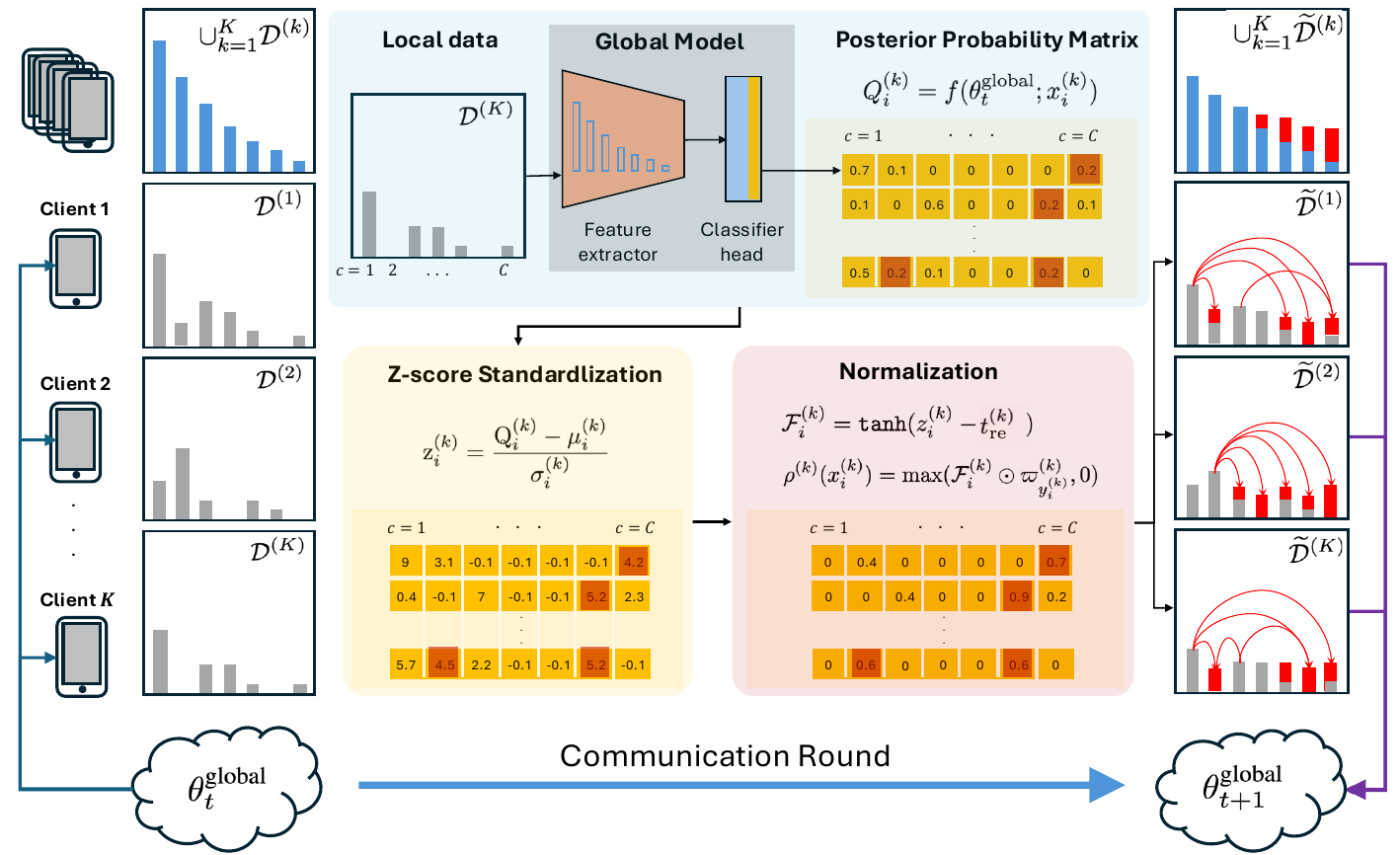}
\caption{FedReLa Framework. At round $t = T_{\text{relabel}}$, FedReLa re-labels the local dataset with the label re-allocator based on the global model before the local training starts.}
\label{fig:frame}
\end{figure*}

\begin{lemma}\label{lem:global_noisy}
    The optimal Bayesian decision boundary of the global aggregated model $\widetilde{\eta}_{j}^{[w]}(x)$ is
    \begin{align}
        \notag \widetilde{S}^{[w]} = \Bigg\{x^{*} & \in \mathcal{X}: \frac{P_j(x^*)}{P_{\ell}(x^*)} \\
        \notag & =  \frac{\sum_{k=1}^{K}w_k\pi_{\ell}^{(k)}[1-2\rho_{\ell\to j}^{(k)}(x)]/\pi_{\ell}}{\sum_{k=1}^{K}w_k\pi_{j}^{(k)}/\pi_j}\cdot\frac{\pi_{\ell}}{\pi_j}\Bigg\}.
    \end{align}
\end{lemma}

Lemma \ref{lem:global_noisy} implies that the label re-allocator balances the global imbalance ratio when
\begin{equation}\label{eq:global_noise_condition}
    \frac{\sum_{k=1}^{K}w_k\pi_{\ell}^{(k)}[1-2\rho_{\ell\to j}^{(k)}(x)]/\pi_{\ell}}{\sum_{k=1}^{K}w_k\pi_{j}^{(k)}/\pi_j} < 1.
\end{equation}
By choosing $w_k=|\mathcal{D}^{(k)}|/|\mathcal{D}|$, $\sum_{k=1}^{K}w_k\pi_{j}^{(k)}/\pi_j=1$, and (\ref{eq:global_noise_condition}) reduces to 
$\sum_{k=1}^{K}w_k\pi_{\ell}^{(k)}[1-2\rho_{\ell\to j}^{(k)}(x)]/\pi_{\ell}<1$, which holds naturally when $\rho_{\ell\to j}^{(k)}(x)>0$ for all $k\in\{1,\ldots,K\}$.
Thus, the label re-allocator can balance both the local and global decision boundary.

Data heterogeneity gives rise to a mismatch between local and global class imbalances in federated learning, where the global minority class becomes the local majority class for some clients. To mitigate the issues caused by such mismatches, we determine the relabeling direction based on local class priors, which avoids unexpected relabeling and enables correction of both global and local decision boundaries. A detailed discussion and formal derivation of this mechanism are presented in Remark \ref{app:remark_data_hetero} of the Appendix.

\section{Framework of FedReLa}\label{sec:algorithm}

Motivated by decision boundary adjustment through data re-labeling, as analyzed in Section \ref{sec:2.2}, we propose FedReLa to mitigate performance degradation caused by data heterogeneity and class imbalance in FL.
FedReLa is a model-agnostic approach, which is designed as a plug-in module that can be seamlessly integrated into any FL algorithm.

As shown in Figure \ref{fig:frame}, 
FedReLa works as a local data one-shot preprocessor between communication rounds of any FL algorithm, with each client applying it locally and in parallel.
Specifically, before client $k$ starts to train the global model $f(\theta;x)$ with parameter $\theta=\theta^{\rm global}_{t}$ received at round $t=T_{\text{relabel}}$, FedReLa re-labels its local dataset $\mathcal{D}^{(k)}=\{(x_{i}^{(k)}, y_{i}^{(k)})\}_{i=1}^{n_k}$ using a client-specific label re-allocator $\rho^{(k)}$, resulting in the re-labeled dataset $\widetilde{\mathcal{D}}^{(k)}=\{(x_{i}^{(k)}, \widetilde{y}_{i}^{(k)})\}_{i=1}^{n_k}=\rho^{(k)}(\mathcal{D}^{(k)})$. Note that the computation of FedReLa only occurs in $T_{\text{relabel}}$, and the re-labeled local dataset $\widetilde{\mathcal{D}}^{(k)}$ can be reused in subsequent training rounds $t > T_{\text{relabel}}$. Thus, the one-shot computations at round $T_{\text{relabel}}$ for label re-allocators are lightweight and almost negligible to the whole training process. We discuss the computational cost of FedReLa in Appendix \ref{apdx:cost}.
Each client then updates the global model locally using $\widetilde{\mathcal{D}}^{(k)}$, and the server aggregates the local updates $\Delta \theta^{(k)}_{t}$ to produce the updated global model with parameter $\theta^{\rm global}_{t+1}$.

The inspiration of FedReLa is to ``reallocate'' the shared feature space that is encroached upon by the majority class (due to biased decision boundaries) to the minority class. 
This is achieved by selectively re-labeling the majority-class samples that intrude into the minority-class feature space with similar features as minority-class samples.
Building upon the analysis in Section \ref{sec:theory}, we let the re-labeling probabilities be proportional to the posterior probabilities of minority classes, and we utilize the global model distributed to each client to perform local inference.
This yields a $|\mathcal{D}^{(k)}|\times|\mathcal{Y}|$ posterior probability matrix $\mathbf{Q}^{(k)}$ for client $k$, where $$Q_i^{(k)} = f(\theta^{\rm global}_{t};x_{i}^{(k)})\in \mathbb{R}^{|\mathcal{Y}|}$$ is the $i$-th row of $\mathbf{Q}^{(k)}$, denoting the posterior probability vector of the $i$-th local instance $x_{i}^{(k)}$.

Crucially, the global model implicitly integrates cross-client discriminative knowledge across all classes $\mathcal{Y} = \{1, 2, \dots, C\}$, making it assign non-zero posterior probabilities even to classes absent from a client's local data.
Due to global imbalance and data heterogeneity, these posterior estimates for minority (tail) classes tend to be systematically biased downward. To address posterior underestimation and obtain a well-calibrated label re-allocator, we introduce two key normalization steps:
(1) z-score Standardization, and
(2) \texttt{tanh} Normalization.

\textbf{Class-wise z-score Standardization:}
Samples near decision boundaries often share features with other classes, thus exhibiting relatively high posterior probabilities for ambiguous ones.
As a result of biased global decision boundaries towards minority classes, most dominant-class samples exhibit vanishingly small posterior probabilities for minority classes.
Despite this, we empirically observe that a non-trivial subset of majority-class samples retains non-negligible probabilities for minority classes---insufficient to trigger misclassification but indicative of proximity to minority-class regions in the feature space.
To better calibrate these underestimated posteriors, particularly for minority classes, we apply class-wise $z$-score standardization, which rescales the posterior distributions within each class.
This highlights candidate samples with shared features for re-labeling.
Specifically, for the $i$-th instance in client $k$, let
\begin{equation}
    \mathcal{I}_{i}^{(k)}=\{i_0\in\{1,\ldots,n_k\}: y_{i_0}^{(k)}=y_{i}^{(k)}\}
\end{equation}
denote the index set of samples in $\mathcal{D}^{(k)}$ that share the same label as the $i$-th instance.
The class-wise mean and standard deviation vectors are computed as
\begin{align}
    \mu_{i}^{(k)} &= \frac{1}{|\mathcal{I}_{i}^{(k)}|}\sum_{i_0\in\mathcal{I}_{i}^{(k)}}Q_{i_0}^{(k)}, \\
    \sigma_{i}^{(k)} &= \sqrt{\frac{1}{|\mathcal{I}_{i}^{(k)}|-1}\sum_{i_0\in\mathcal{I}_{i}^{(k)}}\left(Q_{i_0}^{(k)}-\mu_{i}^{(k)}\right)^2},
\end{align}
where $\mu_i^{(k)} \in \mathbb{R}^{|\mathcal{Y}|}$ and $\sigma_i^{(k)} \in \mathbb{R}^{|\mathcal{Y}|}$ are the class-wise mean and standard deviation vectors of the posterior probabilities over this set.
The $z$-score vector for the $i$-th instance is then given by
\begin{equation}
    \mathrm{z}_i^{(k)} = \frac{Q_{i}^{(k)} - \mu_{i}^{(k)}}{\sigma_{i}^{(k)}}.
    \label{eq:muste}
\end{equation}
As illustrated in Figure~\ref{fig:frame}, the resulting $z$-score matrix $\mathbf{Z}^{(k)}$, with $i$-th row $\mathrm{z}_{i}^{(k)}$, recalibrates $\mathbf{Q}^{(k)}$, amplifying underestimated posterior probabilities of minority (tail) classes and highlighting samples near class boundaries.

\textbf{Normalization via \texttt{tanh}:} 
To ensure that the re-labeling rates in the label re-allocator lie within $[0, 1]$, we rescale the z-scores to the range $[-1, 1]$ using a \texttt{tanh} transformation.
This normalization incorporates two critical components: (1) a client-specific threshold $t_{\rm re}^{(k)}$, which is a tunable hyperparameter that determines the desired re-labeling strength by filtering out samples with weak feature similarity; and (2) a class-wise reweighting vector $\varpi_j^{(k)}\in \mathbb{R}^{|\mathcal{Y}|}$, computed from local class priors to re-label samples asymmetrically.
Specifically, $\varpi_{j}^{(k)}$ reweighs the re-labeling probability from class $j$ to any other class $c \in \mathcal{Y} \setminus \{j\}$ with their class prior difference between $j$ and $c$.

Let $n_{\mathcal{Y}}^{(k)}\in\mathbb{R}^{|\mathcal{Y}|}$ denote the vector of class-wise sample counts in the local dataset $\mathcal{D}^{(k)}$. 
We first apply min-max normalization on $n_{\mathcal{Y}}^{(k)}$ to construct the class-wise reweighting vector $\varpi^{(k)}= 1 - \texttt{minmax}(n^{(k)}_{\mathcal{Y}})$, defined as:
\begin{eqnarray}
    \varpi^{(k)}_{j} = \max(\varpi^{(k)} - \varpi^{(k)}[j],0)
    \label{eq:theta}
\end{eqnarray}
For samples belonging to a local tail class $y_{\text{tail}}$, we have $\varpi^{(k)}[y_{\text{tail}}] = 1$, which implies that $\varpi^{(k)}_{y_{\text{tail}}} = 0$. It zeros the re-labeling probabilities from the minority class to other classes, thereby preserving the integrity of minority-class samples.

Let $\mathbf{Z}^{(k)} \in \mathbb{R}^{N \times C}$ denote the matrix of local $z$-scores on client~$k$, where $\mathbf{Z}_{:,j}^{(k)}$ is its $j$-th column collecting class-$j$ scores over all $N$ local samples.
We use a client-specific, class-wise threshold vector $\mathbf{t}_{\mathrm{re}}^{(k)} \in \mathbb{R}^{C}$, controlled by the hyperparameter $\tau \in [0,1]$.
For each class $j$, the threshold is set to the $(1-\tau)$-quantile of the corresponding local $z$-scores:
\begin{equation}
    t_{\mathrm{re},j}^{(k)} = \mathrm{Quantile}_{1-\tau}\!\left(\mathbf{Z}_{:,j}^{(k)}\right), \quad j = 1, \ldots, C.
    \label{eq:tre_quantile}
\end{equation}
With this threshold, the re-labeling probability for the $i$-th instance on client~$k$ is defined as
\begin{equation}
    \rho^{(k)}\!\left(x_i^{(k)}\right) = \max\!\left(
        \tanh\!\left(z_i^{(k)} - t_{\mathrm{re}}^{(k)}\right) \odot \varpi^{(k)}_{y_i},\, 0
    \right),
    \label{eq:fliprate}
\end{equation}
where $\mathbf{z}_i^{(k)} \in \mathbb{R}^{C}$ is the local $z$-score vector of sample $x_i^{(k)}$, $\odot$ denotes element-wise multiplication, and $\boldsymbol{\varpi}^{(k)}_{y_i}$ is the class-prior weight associated with the observed label $y_i$.
Intuitively, $\mathbf{t}_{\mathrm{re}}^{(k)}$ acts as a tunable, class-wise filter: samples whose $z$-scores are insufficiently large relative to the class-specific threshold receive negative values after the $\tanh(\cdot)$ mapping and are truncated to zero, suppressing re-labeling for instances with weak minority-class similarity.
We provide a sensitivity analysis of $t_{\mathrm{re}}^{(k)}$ in Appendix~\ref{sec:sensitivity}.

Based on the local label re-allocator $\rho^{(k)}$, each client applies probabilistic re-labeling to its data to generate the re-labeled data before local training; this is the only difference FedReLa makes from the standard FL, which adjusts the decision boundaries and thus achieves significant improvement on the performance of minority/tail classes. 
The full procedure is summarized in Algorithm \ref{alg:2}.


\begin{algorithm}[ht]

    \textbf{Input:} local datasets $\mathcal{D}^{(k)}=\{(x_{i}^{(k)}, y_{i}^{(k)})\}_{i=1}^{n_k}$, classifier $f(\cdot;\cdot)$, global model $\theta^{\text{global}}_{t}$\;
    
    \textbf{Parameters:} Threshold $t_{\rm re}^{(k)}$, re-labeling round $T_{\text{relabel}}$\; \\
    \caption{\texttt{Re-Allocator} re-labels local data}\label{alg:2}
    
    \textbf{for} client $k \in K$ at communication round $t = T_{\text{relabel}}$ \textbf{do}: 
    
    \hspace{1.5em}Compute $\varpi^{(k)}$ by (\ref{eq:theta}) with $n^{(k)}_{\mathcal{Y}}$\;
    
    \hspace{1.5em}\textbf{for} $(x_{i}^{(k)},y_{i}^{(k)}) \in \widetilde{\mathcal{D}}^{(k)}$ \textbf{do}:
    
    \hspace{3em} $\mathrm{Q}_i^{(k)} \gets f(\theta^{\text{global}}_t;x_{i}^{(k)})$\;

    
    \hspace{1.5em}\textbf{for} {$(x_i^{(k)},y_i^{(k)}) \in \mathcal{D}^{(k)}$} \textbf{do}:
    
    \hspace{3em} Compute $z_i^{(k)}$ by (\ref{eq:muste}) with $\mathrm{Q}_i^{(k)}$
    
    \hspace{3em} $\varpi^{(k)}_{y_{i}^{(k)}} \gets \texttt{max}(\varpi^{(k)} - \varpi^{(k)}[y_i^{(k)}],0)$
    
    \hspace{3em} Compute $\rho^{(k)}(x_i^{(k)})$ by (\ref{eq:fliprate})
    
    \hspace{3em} $\mathcal{U} \in \mathbb{R}^{|\mathcal{Y}|} \gets \texttt{Bernoulli}(\rho^{(k)}(x_i^{(k)}))$\;
    
    \hspace{3em} \textbf{if} $\mathcal{U}$ contains 1 \textbf{then}:
    
    \hspace{4.5em} $\widetilde{y}_i^{(k)} \gets \mathcal{Y}[\texttt{argmax}(\rho^{(k)}(x_i^{(k)}))]$
    
    \hspace{1.5em} \textbf{return} $\widetilde{\mathcal{D}}^{(k)} \gets \{(x_i^{(k)},\widetilde{y}_i^{(k)})\}_{i=1}^{n_k}$
    
    \end{algorithm}

\paragraph{Robust to biased posteriors.}
To demonstrate how FedReLa calibrates underestimated posterior probabilities and identifies majority-class samples with feature similarities to minority classes, we first clarify the rationale of z-score standardization. The z-score quantifies the degree to which a posterior probability deviates from the distribution of its own class, rather than relying on the absolute value of the posterior. Even when the global model yields biased and severely underestimated posterior probabilities for minority classes, the relative ranking of these deviations across samples remains intact. The standardization step via z-score maps this preserved relative ranking to a standardized range, enabling the selection of samples with the highest feature similarity to minority classes. 

By further incorporating \texttt{tanh} normalization, FedReLa computes well-calibrated re-labeling probabilities from these standardized scores. Consequently, even when posteriors are underestimated by a biased global model, our method can still effectively identify majority-class samples that share salient features with minority classes.

We validate the proposed mechanism by examining re-labeled samples randomly selected from multiple clients’ local data on CIFAR-10-LT, inspecting their underestimated original posterior probabilities and calibrated re-labeling probabilities. As shown in Figure \ref{fig:sample}, these samples exhibit strong feature affinity with their target classes, stemming from both inherently similar categories (e.g., automobile–truck, deer–horse) and instance-level visual similarities even for classes with low overall correlation (e.g., airplane–ship, bird–frog–horse). Despite their generally small underestimated posteriors, these samples are still assigned high re-labeling probabilities, confirming that FedReLa effectively extracts valuable information from biased posterior distributions to guide proper sample re-labeling.

To validate the significance of z-score standardization in FedReLa, we conduct an ablation study comparing its performance using standardized re-labeling probabilities against directly employing underestimated posterior probabilities. Results are presented in Appendix \ref{sec:ablation} Table \ref{tab:zscore}, confirming that z-score standardization facilitates balanced label re-allocation and yields a more favorable trade-off between Head, Medium, and Tail class performance.

\begin{figure}[h]
\centering
\includegraphics[width=1.05\linewidth]{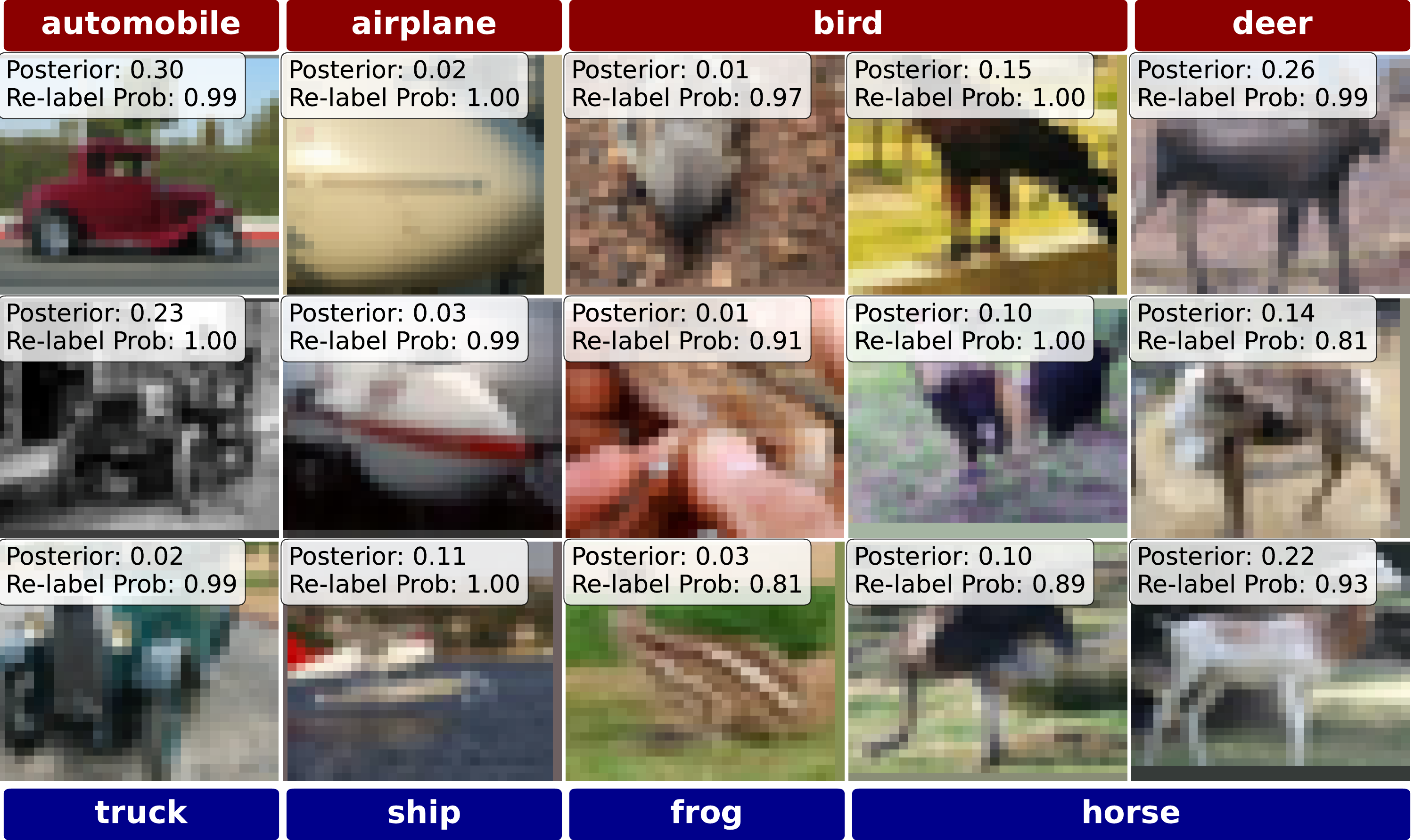}
\caption{FedReLa re-labeling examples on CIFAR-10-LT. Top red: original head classes; bottom blue: re-labeled tail classes. “Posterior”: underestimated tail-class posterior $P(y_{tail}|x)$; “Re-label Prob”: assigned re-labeling probabilities.}
\label{fig:sample}
\end{figure}

\begin{table*}[h]
\centering
\resizebox{\textwidth}{!}{
\begin{tabular}{l|c|c|ccc|ccc}
\hline
\multirow{2}{*}{Dataset} & \multirow{2}{*}{IR} & \multirow{2}{*}{Methods} & \multicolumn{3}{c|}{10\% Minority} & \multicolumn{3}{c}{30\% Minority} \\
\cline{4-9}
& & & Majority & Minority & Overall & Majority & Minority & Overall \\ \hline
\multirow{10}{*}{F-MNIST} & \multirow{5}{*}{10} & FedAvg & 88.43(87.40)\color{black}{-1.03} & 52.50(77.00)\textbf{+24.50} & 84.84(86.36)\textbf{+1.52} & 89.86(90.17)\textbf{+0.31} & 60.67(70.50)\textbf{+9.83} & 81.10(84.27)\textbf{+3.17} \\
& & FedProx & 88.23(87.43)\color{black}{-0.80} & 53.20(76.80)\textbf{+23.60} & 84.73(86.37)\textbf{+1.64} & 90.60(88.14)\color{black}{-2.46} & 60.07(70.40)\textbf{+10.33} & 81.44(82.82)\textbf{+1.38} \\
& & FedNova & 86.81(87.46)\textbf{+0.65} & 67.10(77.50)\textbf{+10.40} & 84.84(86.46)\textbf{+1.62} & 88.41(86.83)\color{black}{-1.58} & 69.77(76.17)\textbf{+6.40} & 82.82(83.63)\textbf{+0.81} \\
& & MOON & 88.41(87.83)\color{black}{-0.58} & 47.00(73.00)\textbf{+26.00} & 84.27(86.35)\textbf{+2.08} & 90.61(89.24)\color{black}{-1.37} & 59.80(69.87)\textbf{+10.07} & 81.37(83.43)\textbf{+2.06} \\
& & CLIMB & 89.05(89.98)\textbf{+0.93} & 65.52(76.24)\textbf{+10.72} & 86.70(88.61)\textbf{+1.91} & 93.00(92.30)\color{black}{-0.70} & 67.23(75.47)\textbf{+8.24} & 85.27(87.25)\textbf{+1.98} \\ \cline{2-9}
& \multirow{5}{*}{20} & FedAvg & 88.64(88.07)\color{black}{-0.57} & 49.00(73.10)\textbf{+24.10} & 84.68(86.57)\textbf{+1.89} & 90.44(87.33)\color{black}{-3.11} & 50.50(75.70)\textbf{+25.20} & 78.46(83.84)\textbf{+5.38} \\
& & FedProx & 89.37(87.77)\color{black}{-1.60} & 44.58(73.60)\textbf{+29.02} & 84.89(86.35)\textbf{+1.46} & 90.69(87.37)\color{black}{-3.32} & 50.00(74.90)\textbf{+24.90} & 78.48(83.63)\textbf{+5.15} \\
& & FedNova & 88.46(88.37)\color{black}{-0.09} & 52.34(71.30)\textbf{+18.96} & 84.85(86.66)\textbf{+1.81} & 85.94(87.27)\textbf{+1.33} & 55.03(77.90)\textbf{+22.87} & 76.67(84.46)\textbf{+7.79} \\
& & MOON & 89.07(88.07)\color{black}{-1.00} & 32.40(66.60)\textbf{+34.20} & 83.40(85.92)\textbf{+2.52} & 91.13(88.63)\color{black}{-2.50} & 44.43(74.77)\textbf{+30.34} & 77.12(84.47)\textbf{+7.35} \\
& & CLIMB & 90.30(90.32)\textbf{+0.02} & 51.28(71.34)\textbf{+20.06} & 86.40(88.43)\textbf{+2.03} & 94.34(90.28)\color{black}{-4.05} & 53.27(73.60)\textbf{+20.33} & 82.02(85.28)\textbf{+3.26} \\ \hline

\multirow{10}{*}{CIFAR-10} & \multirow{5}{*}{10} & FedAvg & 60.10(59.84)\color{black}{-0.26} & 27.80(55.70)\textbf{+27.90} & 56.87(59.43)\textbf{+2.56} & 65.93(62.14)\color{black}{-3.79} & 22.67(41.33)\textbf{+18.66} & 52.95(55.90)\textbf{+2.95} \\
& & FedProx & 60.66(61.18)\textbf{+0.52} & 30.02(58.70)\textbf{+28.68} & 57.60(60.93)\textbf{+3.33} & 67.66(60.79)\color{black}{-6.87} & 22.80(45.10)\textbf{+22.30} & 54.20(56.08)\textbf{+1.88} \\
& & FedNova & 58.54(58.60)\textbf{+0.06} & 29.90(57.00)\textbf{+27.10} & 55.68(58.44)\textbf{+2.76} & 65.23(62.70)\color{black}{-2.53} & 23.47(39.20)\textbf{+15.73} & 52.70(55.65)\textbf{+2.95} \\
& & MOON & 58.62(60.33)\textbf{+1.71} & 17.10(49.30)\textbf{+32.20} & 54.47(59.23)\textbf{+4.76} & 66.67(63.21)\color{black}{-3.46} & 23.63(38.67)\textbf{+15.04} & 53.76(55.85)\textbf{+2.09} \\
& & CLIMB & 81.62(82.68)\textbf{+1.06} & 37.45(46.18)\textbf{+8.73} & 77.20(79.03)\textbf{+1.83} & 86.82(87.47)\textbf{+0.65} & 33.59(43.26)\textbf{+9.67} & 70.85(74.21)\textbf{+3.36} \\ \cline{2-9}

& \multirow{5}{*}{20} & FedAvg & 60.33(60.70)\textbf{+0.37} & 17.25(51.60)\textbf{+34.35} & 56.02(59.79)\textbf{+3.77} & 67.39(61.51)\color{black}{-5.88} & 13.82(47.97)\textbf{+34.15} & 51.32(57.45)\textbf{+6.13} \\
& & FedProx & 59.33(60.58)\textbf{+1.25} & 15.60(53.90)\textbf{+38.30} & 54.96(59.91)\textbf{+4.95} & 67.69(61.96)\color{black}{-5.73} & 15.02(47.87)\textbf{+32.85} & 51.89(57.73)\textbf{+5.84} \\
& & FedNova & 61.77(62.77)\textbf{+1.00} & 26.05(57.80)\textbf{+31.75} & 58.20(62.27)\textbf{+4.07} & 66.97(60.12)\color{black}{-6.85} & 18.20(53.03)\textbf{+34.83} & 52.34(57.99)\textbf{+5.65} \\
& & MOON & 58.75(59.72)\textbf{+0.97} & 10.12(47.70)\textbf{+37.58} & 53.89(58.52)\textbf{+4.63} & 64.51(60.77)\color{black}{-3.74} & 7.65(40.03)\textbf{+32.38} & 47.45(54.55)\textbf{+7.10} \\
& & CLIMB & 79.53(80.34)\textbf{+0.81} & 28.38(40.21)\textbf{+11.83} & 74.42(76.33)\textbf{+1.91} & 87.75(85.81)\color{black}{-1.94} & 24.03(38.77)\textbf{+14.74} & 68.64(71.70)\textbf{+3.06} \\ \hline

\multirow{10}{*}{CIFAR-100} & \multirow{5}{*}{10} & FedAvg & 58.67(58.08)\color{black}{-0.59} & 12.30(23.10)\textbf{+10.80} & 54.03(54.58)\textbf{+0.55} & 58.67(57.07)\color{black}{-1.60} & 14.37(25.70)\textbf{+11.33} & 45.38(47.66)\textbf{+2.28} \\
& & FedProx & 58.14(58.03)\color{black}{-0.11} & 13.00(26.00)\textbf{+13.00} & 53.63(54.83)\textbf{+1.20} & 58.84(58.10)\color{black}{-0.74} & 14.93(22.03)\textbf{+7.10} & 45.67(47.28)\textbf{+1.61} \\
& & FedNova & 58.67(57.90)\color{black}{-0.77} & 13.40(23.90)\textbf{+10.50} & 54.14(54.50)\textbf{+0.36} & 59.49(58.00)\color{black}{-1.49} & 13.53(23.23)\textbf{+9.70} & 45.70(47.57)\textbf{+1.87} \\
& & MOON & 57.55(57.70)\textbf{+0.15} & 13.22(23.92)\textbf{+10.70} & 53.12(54.32)\textbf{+1.20} & 58.60(56.83)\color{black}{-1.77} & 16.37(23.93)\textbf{+7.56} & 45.93(46.96)\textbf{+1.03} \\
& & CLIMB & 47.96(48.28)\textbf{+0.32} & 10.50(24.90)\textbf{+14.40} & 44.21(45.94)\textbf{+1.73} & 49.16(47.44)\color{black}{-1.72} & 10.83(25.87)\textbf{+15.04} & 37.66(40.97)\textbf{+3.31} \\ \cline{2-9}

& \multirow{5}{*}{20} & FedAvg & 59.34(59.01)\color{black}{-0.33} & 6.80(15.80)\textbf{+9.00} & 54.09(54.69)\textbf{+0.60} & 58.73(57.50)\color{black}{-1.23} & 5.90(11.93)\textbf{+6.03} & 42.88(43.83)\textbf{+0.95} \\
& & FedProx & 58.86(58.18)\color{black}{-0.68} & 5.00(14.10)\textbf{+9.10} & 53.47(53.77)\textbf{+0.30} & 59.49(56.93)\color{black}{-2.56} & 6.03(13.23)\textbf{+7.20} & 43.45(43.82)\textbf{+0.37} \\
& & FedNova & 59.16(58.49)\color{black}{-0.67} & 7.30(17.80)\textbf{+10.50} & 53.97(54.42)\textbf{+0.45} & 59.60(58.11)\color{black}{-1.49} & 6.03(13.57)\textbf{+7.54} & 43.53(44.75)\textbf{+1.22} \\
& & MOON & 57.89(57.43)\color{black}{-0.46} & 6.65(18.02)\textbf{+11.37} & 52.77(53.49)\textbf{+0.72}& 59.36(56.14)\color{black}{-3.22} & 5.90(13.87)\textbf{+7.97} & 43.32(43.46)\textbf{+0.14} \\
& & CLIMB & 47.91(47.21)\color{black}{-0.70} & 5.02(16.25) \textbf{+11.23}& 43.62(44.12)\textbf{+0.50}& 49.22(46.42)\color{black}{-2.80} & 4.34(13.44)\textbf{+9.10} & 35.76(36.53)\textbf{+0.77}\\ \hline
\end{tabular}
}
\caption{Test accuracies (in \%) in the format of \texttt{original}\,(+\texttt{FedReLa}){\textbf{+\,\texttt{enhancement}}}\,/{\color{black}\,-\texttt{trade-off}}
of different methods on step-wise imbalance datasets at heterogeneity level of $\alpha = 0.3$.  }
\label{tab:step}
\end{table*}

\section{Experiments}\label{sec:experiments}


\begin{table*}[htp]
    \centering
    \resizebox{\textwidth}{!}{
    \begin{tabular}{c|c|c|cc|cc|cc}
    \hline
    \multirow{2}{*}{Dataset} & \multirow{2}{*}{IF} & Heterogeneity & \multicolumn{2}{c|}{$\alpha=0.1$} & \multicolumn{2}{c|}{$\alpha=0.3$} & \multicolumn{2}{c}{$\alpha=10$} \\ \cline{3-9}
     &  & Method/Metrics & H/M/T-shots & Overall & H/M/T-shots & Overall & H/M/T-shots & Overall \\ \hline
    \multirow{8}{*}{CIFAR-10} & \multirow{4}{*}{50} & FedETF & 85.07/41.03/5.40 & 47.96 & 79.60/67.93/43.10 & 65.15 & 88.77/73.22/56.81 & 74.52 \\
     &  & \multicolumn{1}{r|}{+(FedReLa)} & 78.70/65.20/35.57 & \textbf{61.71} & 74.95/71.57/55.63 & \textbf{68.14} & 86.80/76.62/71.59 & \textbf{79.18} \\
     \cline{3-3}
     &  & FedLOGE & 56.51/58.98/59.74 & 58.22 & 82.81/76.90/67.98 & 76.59 & 88.32/81.32/72.40 & 81.45 \\
     &  & \multicolumn{1}{r|}{+(FedReLa)} & 69.55/70.13/53.09 & \textbf{64.78} & 80.09/78.47/76.86 & \textbf{78.63} & 83.79/83.32/84.12 & \textbf{83.75} \\
     \cline{2-9}
     & \multirow{4}{*}{100} & FedETF & 66.90/36.80/24.40 & 40.87 & 72.05/37.13/25.32 & 42.88 & 92.16/69.52/50.13 & 68.56 \\
     &  & \multicolumn{1}{r|}{+(FedReLa)} & 57.60/44.53/36.95 & \textbf{45.42} & 72.59/40.11/30.47 & \textbf{46.00} & 90.17/70.73/66.91 & \textbf{75.04} \\
     \cline{3-3}
     &  & FedLOGE & 44.93/50.27/41.30 & 45.08 & 86.71/67.52/57.89 & 69.43 & 91.90/75.07/62.50 & 75.09 \\
     &  & \multicolumn{1}{r|}{+(FedReLa)} & 70.60/59.60/57.75 & \textbf{62.16} & 78.62/68.10/69.39 & \textbf{71.77} & 84.58/75.48/79.70 & \textbf{79.90} \\
    \hline
    \multirow{8}{*}{CIFAR-100} & \multirow{4}{*}{50} & FedETF & 62.97/44.65/20.12 & 42.35 & 68.89/46.87/20.28 & 44.83 & 71.11/48.78/18.80 & 44.83 \\
     &  & \multicolumn{1}{r|}{+(FedReLa)} & 56.81/48.60/29.20 & \textbf{44.71} & 61.79/51.20/30.38 & \textbf{47.41} & 56.37/53.09/34.30 & \textbf{47.18} \\
     \cline{3-3}
     &  & FedLOGE & 32.92/37.70/34.64 & 35.09 & 59.40/47.76/31.40 & 45.87 & 68.37/52.31/26.45 & 47.88 \\
     &  & \multicolumn{1}{r|}{+(FedReLa)} & 51.22/46.82/30.37 & \textbf{42.73} & 63.93/51.99/30.04 & \textbf{48.25} & 67.65/51.30/29.77 & \textbf{48.57} \\
     \cline{2-9}
     & \multirow{4}{*}{100} & FedETF & 64.48/42.64/13.72 & 37.11 & 68.75/48.13/14.93 & 39.63 & 71.76/47.48/16.45 & 39.67 \\
     &  & \multicolumn{1}{r|}{+(FedReLa)} & 56.10/46.73/24.14 & \textbf{40.19} & 61.01/51.24/25.02 & \textbf{42.70} & 56.05/50.78/29.86 & \textbf{42.51} \\
     \cline{3-3}
     &  & FedLOGE & 27.00/36.41/25.62 & 29.21 & 50.80/43.67/25.26 & 37.75 & 69.55/50.40/22.98 & 42.87 \\
     &  & \multicolumn{1}{r|}{+(FedReLa)} & 49.25/47.98/23.11 & \textbf{38.08} & 62.36/52.15/23.98 & \textbf{42.89} & 66.28/50.43/25.92 & \textbf{43.36} \\
    \hline
    \end{tabular}
    }
    \caption{Test accuracies (in \%) of different methods on long-tailed CIFAR-10/100.}
    \label{tab:lt}
\end{table*}

\textbf{Datasets:}
To provide a comprehensive evaluation, we conduct experiments under both \textbf{step-wise} and \textbf{long-tailed} global imbalance settings on Fashion-MNIST (F-MNIST) \cite{fashionmnist}, CIFAR-10/100 \cite{cifar10}, and ImageNet-LT (long-tailed) \citep{imagenet} datasets.
For step-wise imbalance, we undersample 10\% or 30\% of the classes with an imbalance ratio (IR) of 10 or 20.
For long-tailed imbalance, the datasets are sampled into a long-tailed class distribution using an imbalance factor (IF) of 50 or 100 as in \citep{LDAM}.
To simulate cross-client heterogeneity, we employ latent Dirichlet sampling \citep{dir} to partition the data in a non-IID fashion across clients. Specifically, we use $K=100$ clients for the step-wise versions of Fashion-MNIST and CIFAR-10, and $K=40$ for their long-tailed versions. For CIFAR-100, we use $K\in\{10,50,100\}$ clients in the long-tailed setting and 10 in the step-wise setting.
The heterogeneity level is controlled by the parameter $\alpha \in \{0.1, 0.3,10\}$. We set the client sample rate to 1.

\textbf{Baseline and prior SOTA:} 
We compare FedReLa with prior baselines and SOTA methods under both step-wise and long-tailed imbalance settings. For step-wise imbalance, we evaluate against FedAvg \citep{fedavg}, FedProx \citep{fedprox}, FedNova \citep{fednova}, MOON \citep{moon}, and CLIMB \citep{climb}. For long-tailed imbalance, we compare with FedETF \citep{Fed-ETF} and the latest SOTA method, FedLOGE \citep{FedLoGe}, and FedYoYo \citep{fedyoyo2025}.
As a data-level method, FedReLa can seamlessly integrate with the above methods, offering further improvements. We thus compare methods trained on original-labeled data with those trained on re-labeled data by FedReLa. All methods are trained with sufficient communication rounds to converge. Please refer to Appendix \ref{apdx:exp_details} for the communication rounds needed to achieve the convergence of each method.

\paragraph{Evaluation Metrics.}
All results are evaluated on class-balanced test sets.
Under step-wise imbalance, we report overall, majority-class, and minority-class test accuracy.
Under long-tailed imbalance, we report overall, head, med, and tail test accuracy, where head, med, and tail classes are defined by the cumulative training-sample proportions of 75\% and 95\% (detailed in Appendix \ref{apdx:eval}).
Results in Table \ref{tab:step} and \ref{tab:lt} are averaged over three runs with different random seeds.

\textbf{Performance comparison:}
For step-wise imbalance scenarios, Table \ref{tab:step} shows that FedReLa consistently enhances accuracy for both minority classes and overall performance across varying imbalance ratios (IR) and minority class proportions at heterogeneity level of $\alpha=0.3$ (see Appendix \ref{sec:ablation} for ablation analysis on $\alpha$). On Fashion-MNIST and CIFAR-10, FedReLa achieves 6.40\%--32.20\% minority-class accuracy improvement and 0.81\%--4.76\% overall accuracy gain under ${\rm IR}=10$. At ${\rm IR}=20$, the approach further elevates minority-class accuracy by 11.83\%--38.30\% and overall accuracy by 1.46\%--7.79\%. On CIFAR-100, FedReLa delivers a steady 6.03\%--15.04\% boost in minority-class accuracy while maintaining overall accuracy superiority. We note slightly degraded majority-class accuracy in the 30\%-minority-class setting, as improved minority-class performance inevitably impairs over-privileged majority-class performance. This aligns with the fundamental trade-off characteristic shared by all imbalanced learning methods. For the 10\%-minority-class scenario, FedReLa exhibits a negligible impact on majority-class accuracy and even improves it on CIFAR-10. This stems from label rectification by FedReLa, which relieves class overlap and thereby reduces outlier-induced interference for majority classes, particularly on clients with local-global IR mismatch. We further conducted additional experiments on FedETF and FedLoGe on step-wise setting, and FedReLa delivers 6.13\% and 13.72\% improvements on overall accuracy at heterogeneity level 0.1 for $IR\in[10,20]$ (See Table \ref{tab:lr_step}).



On long-tailed datasets (Table~\ref{tab:lt}), FedReLa yields consistent gains when built upon FedLOGE, especially under stronger data heterogeneity and more severe class imbalance.
In the most challenging setting (IF=100, $\alpha=0.1$), FedReLa improves overall accuracy by \textbf{+17.08\%} on CIFAR-10-LT and \textbf{+8.87\%} on CIFAR-100-LT over the FedLOGE baseline.
Overall, these results suggest that FedReLa rebalances performance toward tail classes without requiring a uniform sacrifice in head accuracy; in several regimes, both overall and tail accuracy improve simultaneously, which is desirable for long-tailed federated learning.
On CIFAR-100-LT, we further scale the number of clients to $K=50$ and $K=100$; results in Table~\ref{tab:cifar100_lt_clients} show that FedReLa consistently improves overall accuracy, achieves competitive state-of-the-art performance, and remains robust as the client population grows.


\begin{table}[H]
\centering
\begin{tabular}{lcc}
\toprule
\textbf{Method} & \textbf{Overall (\%)} & \textbf{H/M/T (\%)} \\
\midrule
FedYoYo & 38.15 & 41.19/39.42/31.08 \\
+FedReLa & \textbf{38.78} & 40.73/\textbf{40.06/33.71} \\
FedLoGe & 30.52	 & 46.29/28.01/15.02\\
+FedReLa & \textbf{31.70}	
 & 45.43/\textbf{30.44/18.02} \\
\bottomrule
\end{tabular}
\caption{Results on ImageNet-LT (H=Head, M=Medium, T=Tail). 
}
\label{tab:imagenet_lt}
\end{table}

\textbf{Large-Scale Dataset Validation on ImageNet-LT:}
We conduct experiments on ImageNet-LT with data heterogeneity level \(\alpha=0.1\), 20 clients, and 0.4 participation fraction. Another recent SOTA method, FedYoYo \citep{fedyoyo2025}, is used to demonstrate the algorithmic agnosticism of FedReLa. Although FedYoYo falls outside our core comparison scope, as it requires each client to upload the estimated local distribution for aggregation on the server (raising concerns about data privacy), we still include this comparison to demonstrate that FedReLa can consistently enhance performance on the large-scale dataset across various algorithmic methods.

The result in Table \ref{tab:imagenet_lt} confirms that FedReLa’s sample-level re-labeling mechanism avoids the scalability bottlenecks of feature-space methods (e.g., SMOTE) and maintains effectiveness on the large-scale dataset. The consistency of performance gains validates FedReLa’s inherent scalability for real-world large-scale federated learning scenarios.

\begin{table}[H]
    \centering
    \caption{Ablation on normalization strategies for relabeling probability calibration on CIFAR-100 (IF=100). All variants use a 5\% relabeling threshold for fairness.}
    \label{tab:norm_ablation}
    \resizebox{0.9\linewidth}{!}{
    \begin{tabular}{lcccc}
        \toprule
        Strategy & Overall & Head & Medium & Tail \\
        \midrule
        Baseline (FedETF) & 42.81 & 71.73 & 51.44 & 19.82 \\
        Min-max norm      & 43.33 & 70.41 & 52.03 & 21.46 \\
        \textbf{Z-score + tanh} & \textbf{44.54} & 67.53 & \textbf{54.01} & \textbf{24.52} \\
        \bottomrule
    \end{tabular}
    }
\end{table}

\paragraph{Ablation study} To validate superiority over alternatives, we further compare with min-max normalization.
As shown in Table~\ref{tab:norm_ablation}, z-score + tanh achieves the largest gains in both overall and tail-class accuracy (the core goal of imbalanced learning).
Although min-max normalization also improves performance, its gain on tail-class accuracy is limited.
Min-max normalization linearly scales posteriors and is max-value-sensitive, squeezing mildly deviant samples to small relabeling probabilities (e.g., z-scores of 10 and 15 are both mapped to near 1 by tanh, reflecting their high relabeling priority rather than being overshadowed by extreme values). More detailed ablation studies can be found in Appendix \ref{sec:ablation} for (i) the importance of Z-score standardization, (ii) different data-heterogeneity levels, and (iii) comparing one-shot with multi-shot relabeling.

\paragraph{Sensitivity analysis on $T_{\mathrm{relabel}}$.} 
We study how this hyperparameter affects performance on long-tailed CIFAR-10 and CIFAR-100 with IF=100 under two heterogeneity levels ($\alpha=1$ and $\alpha=10$); full results are reported in Table~\ref{tab:trelabel_sensitivity} (Appendix~\ref{sec:sensitivity}).
Across both datasets and heterogeneity settings, FedReLa consistently outperforms the FedETF baseline in overall accuracy and, more importantly, in tail-class accuracy, which is the primary objective under class imbalance.
Once $T_{\mathrm{relabel}}$ reaches 60\% of the total training rounds, further delaying re-labeling yields only marginal changes, indicating that performance is relatively insensitive to $T_{\mathrm{relabel}}$ beyond this point.
Among the choices examined, $T_{\mathrm{relabel}}=80\%R$ achieves the strongest overall and tail-class results in most scenarios.
We therefore adopt $T_{\mathrm{relabel}}\approx 80\%R$ as the default in our main experiments for robust performance across datasets and heterogeneity levels.

\section{Conclusion}

We propose FedReLa, a data-level approach for addressing class imbalance and data heterogeneity in FL. By asymmetrically re-labeling local data by a feature-dependent label re-allocator, FedReLa rectifies decision boundaries without relying on global class priors or additional communication. Empirical results across step-wise and long-tailed settings demonstrate consistent improvements in minority-class and overall accuracy over existing methods, especially under extreme heterogeneity. FedReLa is easy to integrate into algorithmic methods, offering a practical solution for real-world imbalanced federated learning.








\section*{Impact Statement}


This paper presents work whose goal is to advance the field of Machine
Learning. There are many potential societal consequences of our work, none
which we feel must be specifically highlighted here.

\section*{Acknowledgements}
FL is supported by the Australian Research Council (ARC) (Grant No.~DE240101089, LP240100101, DP230101540) and the NSF\&CSIRO Responsible AI program (Grant No.~2303037). MG was supported by ARC DP240102088.
GW was supported by the National Natural Science Foundation of China (Grant No. 12471255), and the Fundamental Research Funds for the Central Universities (Grant No. 63263103). LP is supported by ARC (Grant No.~LP240100101).



\bibliography{example_paper}
\bibliographystyle{icml2026}

\newpage
\appendix
\onecolumn
\section{Technical Details and Proofs}\label{appendix:A}

\begin{table}[h!]
\centering
\resizebox{0.85\textwidth}{!}{
\begin{tabular}{ll}
\toprule
\textbf{Symbol} & \textbf{Definition} \\
\midrule
\multicolumn{2}{l}{\textit{1. Datasets \& Sets}} \\
$\mathcal{D}$ & Global dataset (union of all local datasets) \\
$\mathcal{D}^{(k)}$ & Local dataset of client $k$ \\
$\widetilde{\mathcal{D}}^{(k)}$ & Re-labeled local dataset of client $k$ (by FedReLa) \\
$\mathcal{X}$ & Feature space ($x \in \mathcal{X} \subseteq \mathbb{R}^d$) \\
$\mathcal{Y}$ & Label space ($\mathcal{Y} = \{1, 2, ..., C\}$, $C$: number of classes) \\
$Y^{(k)}$ & Original label set of client $k$ \\
$\widetilde{Y}^{(k)}$ & Re-labeled label set of client $k$ \\
$\mathcal{I}_i^{(k)}$ & Index set of samples in $\mathcal{D}^{(k)}$ with the same label as $x_i^{(k)}$ \\
\midrule
\multicolumn{2}{l}{\textit{2. Model \& Parameters}} \\
$\theta$ & Model parameter vector \\
$\theta_t^{\text{global}}$ & Global model parameter at communication round $t$ \\
$\theta^{(k)}$ & Local model parameter of client $k$ \\
$f(\theta; x)$ & Global model (maps feature $x$ to posterior probabilities) \\
$T_{\text{relabel}}$ & Communication round for FedReLa's one-shot re-labeling \\
\midrule
\multicolumn{2}{l}{\textit{3. Probability \& Distribution}} \\
$\pi_j$ & Global prior probability of class $j$ ($\pi_j = \Pr(Y=j)$) \\
$\pi_j^{(k)}$ & Local prior probability of class $j$ on client $k$ \\
$\pi_j^{[w]}$ & Weighted aggregated prior of class $j$ (server-side) \\
$\eta_j(x)$ & Global posterior probability of class $j$ given $x$ \\
$\eta_j^{(k)}(x)$ & Local posterior probability of class $j$ given $x$ on client $k$ \\
$\widetilde{\eta}_j^{(k)}(x)$ & Posterior probability of class $j$ on $\widetilde{\mathcal{D}}^{(k)}$ \\
$\widetilde{\eta}_j^{[w]}(x)$ & Aggregated posterior probability of class $j$ (server-side) \\
$P_j(x)$ & Class-conditional distribution of $X | Y{=}j$ \\
$P_j^{(k)}(x)$ & Local class-conditional distribution of $X | Y{=}j$ on client $k$ \\
\midrule
\multicolumn{2}{l}{\textit{4. FedReLa Core Parameters}} \\
$\rho_{\ell \to j}^{(k)}(x)$ & Re-labeling probability from local majority class $\ell$ to local minority class $j$ on client $k$ \\
$\rho_{j \to \ell}^{(k)}(x)$ & Re-labeling probability from class $j$ to $\ell$ on client $k$ (set to 0) \\
$\mathbf{Q}^{(k)}$ & Posterior probability matrix of $\mathcal{D}^{(k)}$ ($|\mathcal{D}^{(k)}| \times C$) \\
$z_i^{(k)}$ & Class-wise z-score vector of sample $x_i^{(k)}$ on client $k$ \\
$\mu_i^{(k)}$ & Class-wise mean of posterior probabilities (for z-score) \\
$\sigma_i^{(k)}$ & Class-wise std of posterior probabilities (for z-score) \\
$t_{\text{re}}^{(k)}$ & Client-specific re-labeling threshold (tunable via $\tau$) \\
$\varpi_j^{(k)}$ & Class-wise reweighting vector (from local class priors $\pi^{(k)}$) \\
$n_{\mathcal{Y}}^{(k)}$ & Class-wise sample count vector of $\mathcal{D}^{(k)}$ \\
$\tau$ & Hyperparameter controlling re-labeling strength (top-$\tau\%$ z-scores) \\
\midrule
\multicolumn{2}{l}{\textit{5. Imbalance \& Heterogeneity}} \\
$\mathrm{IR}(\mathcal{D})$ & Global imbalance ratio ($\max_j \pi_j / \min_j \pi_j$) \\
IF & Imbalance factor (for long-tailed datasets) \\
$\alpha$ & Heterogeneity control parameter (Latent Dirichlet Sampling) \\
$K$ & Number of clients in the federation \\
$w_k$ & Aggregation weight of client $k$ (FedAvg: $w_k = |\mathcal{D}^{(k)}| / |\mathcal{D}|$) \\
\midrule
\multicolumn{2}{l}{\textit{6. Decision Boundaries}} \\
$S_{j,\ell}$ & Optimal Bayesian decision boundary between classes $j$ and $\ell$ \\
$\widetilde{S}^{(k)}$ & Decision boundary of client $k$ on re-labeled local dataset $\widetilde{\mathcal{D}}^{(k)}$ \\
\bottomrule
\end{tabular}
}
\caption{Notation Table: Key Symbols and Definitions.}
\label{tab:notation}
\end{table}

\begin{example}\label{example:01}
    {\color{black} We use an extreme example to illustrate that the mismatches between global and local imbalance ratios can amplify the bias in the aggregated decision boundary.}
    Consider a binary classification problem with two classes, $j$ and $\ell$, and two clients, $k_1$ and $k_2$.
    Assume that the global class priors satisfy $\pi_{\ell} \gg \pi_j$, where $\pi_j = m_j / n$, $\pi_{\ell} = m_{\ell}/n$ and $n = |\mathcal{D}^{(k_1)}| + |\mathcal{D}^{(k_2)}|$. 
    Here, $m_j$ and $m_{\ell}$, satisfying $m_{j}+m_{\ell}=n$, are the number of data points in class $j$ and $\ell$, respectively.
    Suppose the local dataset $\mathcal{D}^{(k_1)}$ contains $(m_j - 1)$ samples from the global minority class $j$ and one sample from the global majority class $\ell$, while the local dataset $\mathcal{D}^{(k_2)}$ contains $(m_{\ell} - 1)$ samples from class $\ell$ and one sample from class $j$. The local decision boundaries on $\mathcal{D}^{(k_1)}$ and $\mathcal{D}^{(k_2)}$ are:
    \begin{eqnarray*}
        \left\{x\in \mathcal{X}: \frac{P_j(x)}{P_{\ell}(x)}=\frac{1}{m_j - 1} \right\} \text{~~and~~}
        \left\{x\in \mathcal{X}: \frac{P_j(x)}{P_{\ell}(x)}= m_{\ell} - 1\right\}.
    \end{eqnarray*}
    In FL, consider the global aggregated model $\eta_{j}^{[w]}(x)=w_{k_1}\eta_{j}^{(k_1)}(x)+w_{k_2}\eta_{j}^{(k_2)}(x)$.
    
    Suppose the aggregation weights are chosen as $w_{k_1} \propto |\mathcal{D}^{(k_1)}|$ and $w_{k_2} \propto |\mathcal{D}^{(k_2)}|$, which is widely used in imbalanced classification in the literature. Since $|\mathcal{D}^{(k_1)}| = m_j$ and $|\mathcal{D}^{(k_2)}| = m_{\ell}$, it follows that $w_{k_1} = \pi_j$ and $w_{k_2} = \pi_{\ell}$, implying $w_{k_1} \ll w_{k_2}$. In addition, the local imbalance ratios are $\mathrm{IR}(\mathcal{D}^{(k_1)}) = 1/(m_j-1)$ and $\mathrm{IR}(\mathcal{D}^{(k_2)}) = m_{\ell} - 1$, so that $\mathrm{IR}(\mathcal{D}^{(k_1)})\ll \mathrm{IR}(\mathcal{D}^{(k_2)})$. 
    Thus, the decision boundary of the global aggregated model $\eta_{j}^{[w]}(x)$ is $S_{j,\ell}^{[w]} =  \{x\in \mathcal{X}: P_j(x)/P_{\ell}(x)=\pi_{\ell}^{[w]}/\pi_j^{[w]}\} $ with $\pi_{\ell}^{[w]}=\pi_j(m_j-1)/m_j+\pi_{\ell}/m_{\ell}$ and $\pi_j^{[w]}=\pi_j/m_j+\pi_{\ell}(m_{\ell}-1)/m_{\ell}$.
    As a result, during model aggregation in each communication round, the global imbalance is exacerbated due to the dominant contribution from client $k_2$, amplified by both its large aggregation weight $w_{k_2}$ and local imbalance ratio $\mathrm{IR}(\mathcal{D}^{(k_2)})$. 

    Even under the uniform averaging with $w_{k_1}=w_{k_2}=1/2$, the decision boundary of the global aggregated model $\eta_{j}^{[w]}(x)$ is $S_{j,\ell}^{[w]} =  \{x\in \mathcal{X}: P_j(x)/P_{\ell}(x)=\pi_{\ell}^{[w]}/\pi_j^{[w]}\} $ with $\pi_{\ell}^{[w]}=(m_j-1)/(2m_j)+1/(2m_{\ell})$ and $\pi_j^{[w]}=1/(2m_j)+(m_{\ell}-1)/(2m_{\ell})$. The decision boundary is still biased due to the global imbalance.
\end{example}

\medskip

\begin{remark} \label{app:remark_data_hetero}
    Due to data heterogeneity, local class distribution can deviate significantly from the global one.
    It is possible for a class that is globally a minority to become a majority within certain clients.
    As local clients lack access to the global class prior ratios, the mismatch can lead to re-labeling in unexpected directions. For instance, when re-labeling class $j$ samples to class $\ell$ even if $\pi_{\ell} \gg \pi_j$ globally.
    To handle this, we let the re-labeling direction be determined by local priors: on client $k$, if $\pi_{\ell}^{(k)} > \pi_{j}^{(k)}$, then class $\ell$ samples are re-labeled to class $j$, and vice versa. This results in the following Bayesian decision boundary on client $k$:
    \begin{align}
        \notag \widetilde{S}^{(k)} =  \Bigg\{x^{*} \in \mathcal{X}: \frac{P_j(x^*)}{P_{\ell}(x^*)} = \frac{1-2\rho_{\ell\to j}^{(k)}(x^*)\cdot\mathbb{I}(\pi_{\ell}^{(k)}>\pi_{j}^{(k)})}{1-2\rho_{j\to\ell}^{(k)}(x^*)\cdot\mathbb{I}(\pi_{\ell}^{(k)}<\pi_{j}^{(k)})}\cdot \frac{\pi_{\ell}^{(k)}}{\pi_{j}^{(k)}}\Bigg\}.
    \end{align}
    Then, the Bayesian decision boundary of the global aggregated model $\widetilde{\eta}_{j}^{[w]}(x)$ takes the form:
    \begin{align}
        \notag \widetilde{S}^{[w]} = \Bigg\{x^{*} \in \mathcal{X}: \frac{P_j(x^*)}{P_{\ell}(x^*)} = \frac{\sum_{k=1}^{K}w_k\pi_{\ell}^{(k)}[1-2\rho_{\ell\to j}^{(k)}(x)]\mathbb{I}(\pi_{\ell}^{(k)}>\pi_{j}^{(k)})/\pi_{\ell}}{\sum_{k=1}^{K}w_k\pi_{j}^{(k)}[1-2\rho_{j\to\ell}^{(k)}(x)]\mathbb{I}(\pi_{\ell}^{(k)}<\pi_{j}^{(k)})/\pi_j}\cdot\frac{\pi_{\ell}}{\pi_j}\Bigg\}.
    \end{align}
    Under global imbalance where $\pi_{\ell} \gg \pi_j$, we typically observe that $\sum^K_{k = 1}\mathbb{I}(\pi_{\ell}^{(k)}>\pi_{j}^{(k)}) > \sum^K_{k = 1}\mathbb{I}(\pi_{\ell}^{(k)}<\pi_{j}^{(k)})$, 
    meaning more clients locally reflect the global imbalance than contradict it.
    Furthermore, even if $\pi_{\ell}^{(k_0)} < \pi_{j}^{(k_0)}$ for some client $k_0$, its weight $w_{k_0} \propto |\mathcal{D}^{(k_0)}|$ is often small as $|\mathcal{D}^{(k_0)}|$ is less than double of the total number of class-$j$ samples in the full dataset $\mathcal{D}$.
    As a result, we still expect a correction in the decision boundary of the global aggregated model with
    $\frac{\sum_{k=1}^{K}w_k\pi_{\ell}^{(k)}[1-2\rho_{\ell\to j}^{(k)}(x)]\mathbb{I}(\pi_{\ell}^{(k)}>\pi_{j}^{(k)})/\pi_{\ell}}{\sum_{k=1}^{K}w_k\pi_{j}^{(k)}[1-2\rho_{j\to\ell}^{(k)}(x)]\mathbb{I}(\pi_{\ell}^{(k)}<\pi_{j}^{(k)})/\pi_j} < 1$.
\end{remark}

\paragraph{Explanation on Aggregated Model Representation}
To analyze how re-labeling influences the global decision boundary, we adopt the global model defined via posterior aggregation (i.e., $\sum_{k=1}^{K} w_k f(x, \theta^{(k)})$, where $f(x, \theta^{(k)})$ denotes the local posterior of client $k$ with parameter $\theta^{(k)}$) instead of parameter aggregation (i.e., $f\left(x, \sum_{k=1}^{K} w_k \theta^{(k)}\right)$). This choice is motivated by two key considerations: (1) the posterior-aggregated form renders changes in the decision boundary more explicit and easier to quantify, which aligns with our focus on analyzing re-labeling's effect; (2) it is consistent with statistical model averaging ideas, providing a flexible framework for heterogeneous FL scenarios.

No specific constraints are imposed on the aggregation weights $w_k$, and our only assumption is that the aggregated global model can be expressed as a weighted average of local posteriors, which we explicitly formalize. Furthermore, the two aggregation paradigms (parameter-aggregated and posterior-aggregated) are approximately equivalent under mild regularity conditions, as justified by first-order Taylor expansion:

Assume all local parameters $\theta^{(k)}$ are sufficiently close to a common reference value $\theta_0$ (a reasonable condition in late-stage FL training when models converge). Expanding both models around $\theta_0$:
\begin{enumerate}
        \item For the parameter-aggregated global model:
        \begin{align*}
        f\left(x, \sum_{k=1}^{K} w_k \theta^{(k)}\right) &\approx f(x,\theta_0) + \left.\frac{\partial f(x,\theta)}{\partial\theta}\right|_{\theta=\theta_0} \sum_{k=1}^{K} w_k (\theta^{(k)} - \theta_0) \\
        &= f(x,\theta_0) + \sum_{k=1}^{K} w_k \left.\frac{\partial f(x,\theta)}{\partial\theta}\right|_{\theta=\theta_0} (\theta^{(k)} - \theta_0)
        \end{align*}
        \item For the posterior-aggregated global model (noting $\sum_{k=1}^{K} w_k = 1$):
        \begin{align*}
        \sum_{k=1}^{K} w_k f(x, \theta^{(k)}) &\approx f(x,\theta_0) + \left.\frac{\partial f(x,\theta)}{\partial\theta}\right|_{\theta=\theta_0} \left(\sum_{k=1}^{K} w_k \theta^{(k)} - \theta_0\right) \\
        &= f(x,\theta_0) + \sum_{k=1}^{K} w_k \left.\frac{\partial f(x,\theta)}{\partial\theta}\right|_{\theta=\theta_0} (\theta^{(k)} - \theta_0)
        \end{align*}
\end{enumerate}

The two expansions are identical, confirming that the parameter-aggregated and posterior-aggregated global models are \textbf{first-order equivalent} when local parameters are sufficiently close. This justifies our use of the posterior-aggregated form for analyzing decision boundary changes, as it does not introduce substantive deviations from standard parameter-aggregated FL while offering greater analytical tractability.

\subsection{Proof of Lemma \ref{lem:bound_noise_k}}

\begin{proof}
    As we are considering the binary classification setting, the optimal Bayesian decision boundary based on $\widetilde{\mathcal{D}}^{(k)}$ is 
    \begin{equation}
        \notag \widetilde{S}^{(k)} = \left\{x^*\in\mathcal{X}: \widetilde{\eta}_{j}^{(k)}(x^*) = \widetilde{\eta}_{\ell}^{(k)}(x^*)\right\},
    \end{equation}
    where 
    \begin{equation}
        \notag \widetilde{\eta}_{\ell}^{(k)}(x) = \eta_{\ell}^{(k)}(x)[1 - \rho_{\ell\to j}^{(k)}(x)] + \eta_{j}^{(k)}(x)\rho_{j\to\ell}^{(k)}(x).
    \end{equation}
    Given the formulation of $\widetilde{\eta}_{j}^{(k)}(x)$, we need
    \begin{equation}
        \notag \eta_{j}^{(k)}(x^*)[1 - \rho_{j\to\ell}^{(k)}(x^*)] + \eta_{\ell}^{(k)}(x^*)\rho_{\ell\to j}^{(k)}(x^*) = \eta_{\ell}^{(k)}(x^*)[1 - \rho_{\ell\to j}^{(k)}(x^*)] + \eta_{j}^{(k)}(x^*)\rho_{j\to\ell}^{(k)}(x^*),
    \end{equation}
    which is equivalent to
    \begin{equation}
        \notag \eta_{j}^{(k)}(x^*)[1 - 2\rho_{j\to\ell}^{(k)}(x^*)] = \eta_{\ell}^{(k)}(x^*)[1 - 2\rho_{\ell\to j}^{(k)}(x^*)].
    \end{equation}
    Regarding the fact that $\eta_{j}^{(k)}(x) = \pi_{j}^{(k)}P_j(x)/\sum_{j_0\in\mathcal{Y}}\pi_{j_0}^{(k)}P_{j_0}(x)$, the above equation can be simplified to
    \begin{equation}
        \notag \pi_{j}^{(k)}P_j(x^*)[1 - 2\rho_{j\to\ell}^{(k)}(x^*)] = \pi_{\ell}^{(k)}P_{\ell}(x^*)[1 - 2\rho_{\ell\to j}^{(k)}(x^*)],
    \end{equation}
    and the final result follows immediately.
\end{proof}

\subsection{Proof of Lemma \ref{lem:global_noisy}}

\begin{proof}
    The global aggregated model satisfies
    \begin{equation}
        \notag \widetilde{\eta}_{j}^{[w]}(x) = \sum_{k=1}^{K}w_k\widetilde{\eta}_{j}^{(k)}(x) = \sum_{k=1}^{K}w_k \left\{\eta_{j}^{(k)}(x) + \eta_{\ell}^{(k)}(x)\rho_{\ell\to j}^{(k)}(x)\right\}
    \end{equation}
    and
    \begin{equation}
        \notag \widetilde{\eta}_{\ell}^{[w]}(x) = \sum_{k=1}^{K}w_k\widetilde{\eta}_{\ell}^{(k)}(x) = \sum_{k=1}^{K}w_k \eta_{\ell}^{(k)}(x)[1 - \rho_{\ell\to j}^{(k)}(x)].
    \end{equation}
    Then, for $x^*$ on the Bayesian decision boundary, it requires that
    \begin{align}
        \notag \widetilde{\eta}_{j}^{[w]}(x^*) = &  \sum_{k=1}^{K}w_k \left\{\eta_{j}^{(k)}(x^*) + \eta_{\ell}^{(k)}(x^*)\rho_{\ell\to j}^{(k)}(x^*)\right\} \\
        \notag = & \sum_{k=1}^{K}w_k \eta_{\ell}^{(k)}(x^*)[1 - \rho_{\ell\to j}^{(k)}(x^*)] = \widetilde{\eta}_{\ell}^{[w]}(x^*),
    \end{align}
    which can be simplified to
    \begin{equation}
        \notag \sum_{k=1}^{K}w_k\eta_{j}^{(k)}(x^*) = \sum_{k=1}^{K}w_k \eta_{\ell}^{(k)}(x^*)[1 - 2\rho_{\ell\to j}^{(k)}(x^*)].
    \end{equation}
    Applying $\eta_{j}^{(k)}(x) = \pi_{j}^{(k)}P_j(x)/\sum_{j_0\in\mathcal{Y}}\pi_{j_0}^{(k)}P_{j_0}(x)$ again, we get the desired result.
\end{proof}

\section{Additional Experiment Details}\label{apdx:exp_details}
 The code is available at: \url{https://github.com/guangzhengh/FedReLa.git}.
 
\paragraph{Training details.}
To ensure fair comparison, all global models are trained until full convergence with communication rounds adapted per method. Specifically, baseline methods require 500 rounds for convergence, while CLIMB, which introduces class-wise loss reweighting parameters, demands extended training: 2000 rounds on Fashion-MNIST and CIFAR-10, and 1000 rounds on CIFAR-100. As we do not intend to compare these algorithm-level methods, we use the SGD optimizer with the same weight decay and momentum as they reported in their original implementations: weight decay of $0.00001$ and momentum 0.9 for all methods except long-tailed-oriented methods FedETF and FedLOGE, which follow their original implementations with zero weight decay and momentum 0.5. All experiments were conducted with three distinct random seeds, and their average results are reported in the tables. 

\paragraph{Evaluation metrics.} \label{apdx:eval}
A balanced test dataset is used to evaluate the overall accuracy performance of the global model. Additionally, the average test accuracy for both minority and majority classes is reported for the step-wise imbalanced setting. For long-tailed datasets, we report the accuracy over head, medium, and tail classes as Many-, Medium-, and Few-shot, respectively. 

Adhering to the long-tailed federated learning protocol established in \citep{FedLoGe}, we categorize classes into three disjoint subsets based on sample size distribution: head (majority), medium, and tail (minority) class groups, constituting 75\%, 20\%, and 5\% of total samples, respectively. To evaluate model performance through stratified accuracy metrics, we report Head/Medium/Tail-shot accuracies corresponding to these partitions in Table \ref{tab:lt}. 

\subsection{Computational cost} \label{apdx:cost}
All experiments were conducted on a Spartan cluster on a single node equipped with one NVIDIA H100 GPUs (80GB memory)
10GB RAM with 12 CPU cores.

Before approximating the computational cost of FedReLa, we would like to clarify the fundamental difference between \textbf{extra Local training} and \textbf{extra local computation}:

\begin{enumerate}
    \item Local training overhead involves \textbf{gradient updates for new parameters or module}. For example, methods that introduce new optimizable parameters (e.g., CLIMB, FedLOGE, etc.) require extra per-round local training overhead to update the gradients of these parameters.
    \item ONE-TIME Model Inference: FedReLa only performs \textbf{one-time model inference during a single round} to obtain posterior probabilities, without updating the model or gradients. Therefore, we describe FedReLa as operating "without extra local training."
\end{enumerate}

\paragraph{Approximate one-time computation cost of FedReLa.}
The strength of FedReLa as a data-level method lies in its requirement for only a single computational step during a single round to refine the imbalanced data distribution, thereby achieving long-lasting improvements in model performance.
The core operation of FedReLa is model inference (forward pass) to obtain the local posterior probability matrix $Q^{(k)}$. We approximately consider
\begin{eqnarray*}
    &\mathrm{FLOP}_{\text{train}} = \mathrm{FLOP}_{\text{forward}} +\mathrm{FLOP}_{\text{backpropagation}},\\
    &\mathrm{FLOP}_{\text{backpropagation}} \approx 2 \times \mathrm{FLOP}_{\text{forward}}.
\end{eqnarray*}
Thus, the FLOPs required for $Q^{(k)} = f(\theta^{\text{global}}_{T_{\text{relabel}}}; X^{(k)})$ can be approximately quantified with:
\begin{eqnarray*}
    \mathrm{FLOP}_{Q^{(k)}} = \mathrm{FLOP}_{\text{forward}} \approx \frac{1}{3} \mathrm{FLOP}_{\text{train}}.
\end{eqnarray*}
The total computation cost of FedReLa is approximately 1/3 of the computation cost of a single training round. This cost is One-Time only during the single round of $T_{\text{relabel}}$. More importantly, the local posterior probabilities can be naturally collected during training epochs. \textbf{This allows the computation cost of this one-time model inference to be merged with the normal training overhead.}

\paragraph{Runtime comparison.} Unlike methods requiring from-scratch training, FedReLa enhances classifier performance solely through one-shot re-labeling during the fine-tuning phase. Consequently, its computational overhead is primarily determined by the base federated learning algorithm it augments. For instance, each communication round of FedLOGE requires an average of 72.36 seconds on CIFAR100. When FedReLa enhances FedLOGE with a one-shot computation for label re-allocator, the average communication round time increased to 73.06 seconds, which is negligible.

\subsection{Additional experiment results}
\paragraph{Additional experiment on step-wise setting with recent SOTAs.} Although recent methods, such as FedETF \citep{Fed-ETF} and FedLOGE \citep{FedLoGe}, are long-tail-oriented approaches, we conducted additional experiments on CIFAR-10 with step-wise imbalance. 
The results in Table \ref{tab:lr_step} demonstrate that FedReLa still achieves SOTA performance on step-wise imbalance. FedReLa brings significant improvements, especially under higher imbalance ratios and more heterogeneous data.
\begin{table}[H]
\centering
\setlength{\tabcolsep}{3pt} 
\begin{tabular}{l|c|cc|cc}
\hline
\multicolumn{2}{c}{\textbf{}} & \multicolumn{2}{c}{$\alpha$ = 0.3} & \multicolumn{2}{c}{$\alpha$ = 0.1} \\
\hline
IR & Method & Minority/Majority & Overall & Minority/Majority & Overall \\
\hline
\multirow{2}{*}{10} & FedETF & 74.21/93.79 & 84.01 & 44.14/96.46 & 70.30 \\
& \textbf{+FedReLa} & 82.11/91.72 & 86.92 & \textbf{68.28/84.52} & \textbf{76.43} \\
& FedLOGE & 80.49/92.12 & 86.32 & 61.57/86.05 & 73.81 \\
& \textbf{+FedReLa} & \textbf{85.81/89.73} & \textbf{87.77} & 68.7/81.64 & 75.17 \\
\hline
\multirow{2}{*}{20} & FedETF & 69.31/87.73 & 78.52 & 43.65/72.55 & 58.10 \\
& \textbf{+FedReLa} & \textbf{82.75/83.74} & \textbf{83.11} & \textbf{67.85/75.75} & \textbf{71.82} \\
& FedLOGE & 79.30/85.9 & 82.60 & 45.61/63.79 & 54.70 \\
& +FedReLa & 79.60/84.21 & 81.90 & 59.18/70.73 & 64.95 \\
\hline
\end{tabular}
\caption{Performance on Step-wise-imbalanced CIFAR-10.}
\label{tab:lr_step}
\end{table}

\paragraph{Additional experiment on higher proportion of minority classes.} In addition to 10\% and 30\% minority classes for step-wise-imbalanced datasets, we further extend the proportion to 50\% to examine the consistency of enhancement from FedReLa on extreme conditions. As demonstrated in Table \ref{tab:ir10_50percent}, FedReLa delivers significant performance gains even in the extreme case where minority classes constitute 50\% of the data. Without the FedReLa boost, baseline methods exhibit pronounced accuracy degradation as the proportion of the minority class increases. Our proposed label re-allocator effectively mitigates this performance deterioration while simultaneously enhancing overall accuracy. These results strongly validate FedReLa's capability to provide robust performance enhancements for federated learning methods that face substantial minority class presence.

\begin{table}[H]
\centering
\resizebox{1\textwidth}{!}{%
\begin{tabular}{l|cc|cc|cc}
\hline
\multicolumn{7}{c}{IR=10 with 50\% Minority Classes} \\ \hline
 & \multicolumn{2}{c|}{Fashion-MNIST} & \multicolumn{2}{c|}{CIFAR-10} & \multicolumn{2}{c}{CIFAR-100} \\ \hline
 Method & Minority & Overall & Minority & Overall &Minority & Overall\\ \hline
FedAvg & 63.44(77.10)\textbf{+13.66} & 78.29(82.90)\textbf{+4.61} & 16.94(44.42)\textbf{+27.48} & 48.67(56.72)\textbf{+8.05} & 12.17(25.27)\textbf{+13.10} & 42.48(46.76)\textbf{+4.28} \\ \hline
FedProx & 63.78(77.24)\textbf{+13.46} & 78.62(82.62)\textbf{+4.00} & 16.32(43.54)\textbf{+27.22} & 48.20(56.52)\textbf{+8.32} & 12.83(23.13)\textbf{+10.30} & 41.97(46.28)\textbf{+4.31} \\ \hline
FedNova & 77.60(80.50)\textbf{+2.90} & 81.89(83.30)\textbf{+1.41} & 23.47(39.20)\textbf{+15.73} & 52.70(55.65)\textbf{+2.95} & 13.53(23.23)\textbf{+9.70} & 45.70(47.57)\textbf{+1.87} \\ \hline
MOON & 64.68(76.60)\textbf{+11.92} & 79.22(82.74)\textbf{+3.52} & 11.20(38.54)\textbf{+27.34} & 46.35(54.69)\textbf{+8.34} & 11.17(24.23)\textbf{+13.06} & 42.53(46.56)\textbf{+4.03} \\ \hline
CLIMB & 75.47(79.90)\textbf{+4.43} & 85.22(87.14)\textbf{+1.92} & 31.35(42.85)\textbf{+11.50} & 69.68(72.01)\textbf{+2.33} & 11.00(24.48)\textbf{+13.48} & 30.71(34.79)\textbf{+4.08} \\ \hline
\end{tabular}%
}
\caption{Test accuracies (in \%) of different methods on step-wise imbalance datasets under IR=10 with 50\% minority classes at heterogeneity level $\alpha = 0.3$ in the format of \texttt{original}\,(+\texttt{FedReLa}){\textbf{\,+\,\texttt{enhancement}}}.}
\label{tab:ir10_50percent}
\end{table}

\paragraph{Large number of clients on CIFAR100-LT}
We extend CIFAR100-LT experiments to 50 and 100 clients, with an imbalance factor \(imb\_factor=100\) and data heterogeneity \(\alpha=0.1\). The participation fractions are set to 0.2 and 0.1, respectively. As the number of clients increases, class absences become increasingly severe. This setup aims to verify FedReLa’s robustness to extreme class imbalance and its compatibility with diverse algorithmic approaches.

\begin{table}[H]
\centering
\resizebox{\linewidth}{!}{
\begin{tabular}{lcccc}
\toprule
\textbf{Method} & \multicolumn{2}{c}{\textbf{50 Clients}} & \multicolumn{2}{c}{\textbf{100 Clients}} \\
\cmidrule(lr){2-3} \cmidrule(lr){4-5}
& \textbf{Overall Accuracy (\%)} & \textbf{H/M/F Accuracy (\%)} & \textbf{Overall Accuracy (\%)} & \textbf{H/M/F Accuracy (\%)} \\
\midrule
FedLC \citep{fedlc} & 32.81 & 56.20/32.42/9.74 & 23.72 & 46.41/20.74/4.23 \\
+FedReLa & \textbf{34.76} & 49.92/\textbf{35.24/17.21} & \textbf{25.13} & 39.82/\textbf{28.34/7.02} \\
\midrule
FedYoYo \citep{fedyoyo2025} & 40.89 & 54.32/41.95/24.12 & 30.73 & 34.12/29.64/27.92 \\
+FedReLa & \textbf{41.31} & 53.62/\textbf{42.24/25.91} & \textbf{32.13} & 33.90/\textbf{32.72/29.22} \\
\midrule
FedETF & 31.89 & 60.67/40.45/9.22 & 28.71 & 59.00/33.72/9.73 \\
+FedReLa & \textbf{33.94} & 55.12/\textbf{43.61/15.12} & \textbf{31.50} & 52.21/\textbf{38.73/16.53} \\
\midrule
FedLoGe & 34.83 & 57.12/42.21/17.29 & 33.08 & 62.71/38.78/13.75 \\
+FedReLa & \textbf{35.62} & 57.01/\textbf{44.32/18.00} & \textbf{34.32} & 59.90/\textbf{42.22/16.00} \\
\bottomrule
\end{tabular}
}
\caption{Performance comparison on CIFAR100-LT with 50/100 clients (H=Head, M=Medium, F=Few-shot). FedReLa consistently boosts few-shot accuracy across baselines.}
\label{tab:cifar100_lt_clients}
\end{table}

Table \ref{tab:cifar100_lt_clients} demonstrates three key conclusions: (1) FedReLa delivers consistent enhancements for all baselines; (2) Even with 100 clients (a large-scale client setup), FedReLa maintains performance gains, validating its scalability to distributed environments with numerous clients and its robustness to severe class absence; (3) The consistent improvements across diverse algorithmic paradigms further confirm FedReLa’s algorithm-agnostic property as a data-level plug-in.

\paragraph{Tail-Class Accuracy Gain on CIFAR10-LT Under Extreme Heterogeneity}
Under \(\alpha=0.1\), ~70\% of tail classes are absent from clients, simulating extreme real-world heterogeneity.

\begin{table}[H]
\centering
\resizebox{0.7\textwidth}{!}{
\begin{tabular}{ll|ccc|ccc}
\toprule
\textbf{Method} & \textbf{IF} &
\multicolumn{3}{c|}{\textbf{Med+Tail Gain (\%)}} &
\multicolumn{3}{c}{\textbf{Overall Gain (\%)}} \\
\cmidrule(lr){3-5} \cmidrule(lr){6-8}
+FedReLa & & $\alpha{=}0.1$ & $\alpha{=}0.3$ & $\alpha{=}10$ & $\alpha{=}0.1$ & $\alpha{=}0.3$ & $\alpha{=}10$ \\
\midrule
FedETF & 50 & \textbf{+54.34} & +16.17 & +18.18 & \textbf{+13.75} & +2.99 & +4.66 \\
FedETF & 100 & \textbf{+20.28} & +8.13 & +17.99 & +4.55 & +3.12 & \textbf{+6.48} \\
\midrule
FedLoGe & 50 & +4.50 & +10.45 & \textbf{+13.72} & \textbf{+6.56} & +2.04 & +2.30 \\
FedLoGe & 100 & \textbf{+25.78} & +12.08 & +17.61 & \textbf{+17.08} & +2.34 & +4.81 \\
\bottomrule
\end{tabular}
}
\caption{Accuracy gains of +FedReLa on CIFAR-10-LT vs.\ the matching baseline in Table~\ref{tab:lt}. Med+Tail gain is the sum of medium- and tail-class improvements.}
\label{tab:tail_improvement}
\end{table}

As presented in Table~\ref{tab:tail_improvement}, FedReLa yields its largest accuracy improvements under the most extreme long-tailed and heterogeneous settings, with gains in both overall and medium-/tail-class performance becoming markedly more pronounced as the imbalance factor increases and client data grow more non-IID. In contrast, improvements are comparatively modest under milder regimes where the long-tail effect is weaker and client distributions are closer to IID. These results confirm that FedReLa is explicitly designed for federated scenarios with severe class absence and strong heterogeneity: its deferred re-labeling strategy and calibrated posterior estimation enable selective identification of majority-class samples that resemble absent minorities, avoiding blind re-labeling and delivering robust performance gains where baseline methods struggle most.

\subsection{Sensitivity analysis}\label{sec:sensitivity}

We perform the sensitivity analysis of re-labeling threshold $t_{\rm re}^{(k)}$ on long-tailed CIFAR-10 with $\rm IF = 50$. On each client, the class-wise threshold $t_{\rm re}^{(k)}$ is determined by the top-$\tau_{\rm}$ \% of z-scores. The threshold $t_{\rm re}^{(k)}$ controls the re-labeling strength (the amount of re-labeled samples) as demonstrated in Figure \ref{fig:sensitive}(a). This serves as a safeguard to regulate the number of samples re-labeled by FedReLa. For instance, using the top 1\% z-score as the re-labeling threshold limits the number of re-labeled samples to be less than $1\%$ of local data. 
Figure \ref{fig:sensitive}(a) shows that the amount of re-labeled samples scales linearly with top-$\tau$ percentiles. 

In Figure \ref{fig:sensitive}(b), when $\tau \leq 5$, tail-class performance gains outweigh head-class losses. When $\tau> 5$, medium-class accuracy steadily improves and head-class accuracy continues to decline slowly, while tail-class accuracy remains relatively stable. The effect on performance of turning $t_{\rm re}^{(k)}$ up reveals that: (1) Initially, re-labeled head-class samples with a small $\tau$ mostly invade tail-class feature space. (2) After re-labeling these critical samples, further label re-allocating relieves the head-class invasion of the medium-class feature space. (3) FedReLa prioritizes re-labeling samples that most severely invade tail-class regions. We observe similar results on the CIFAR100-LT (Table \ref{tab:sensitivity} in the appendix). 

\begin{figure}[h!]
\centering
\includegraphics[width=0.95\textwidth]{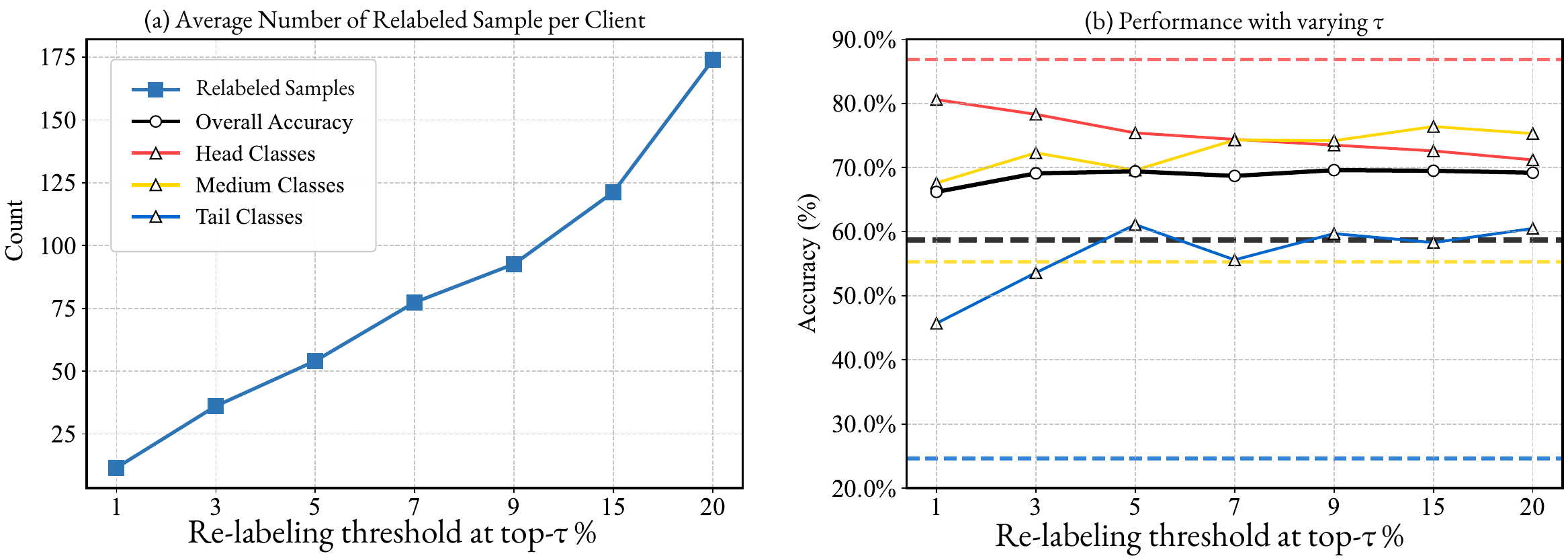}
\caption{Sensitive analysis respect to $\tau$, which controls the re-labeling strength.}
\label{fig:sensitive}
\end{figure}

The threshold-tuning capability allows FedReLa to deliver customized class-wise enhancement, prioritizing tail-class gains ($\tau = 5$) while preserving overall performance. This strategic trade-off (suppressing overprivileged head classes to boost tails) is a unique advantage over static algorithm-level approaches \citep{Fed-ETF, FedLoGe}, as evidenced by the accuracy curves surpassing the baseline (dashed lines) in critical regions. In practice, we can tune the trade-off through $\tau$ depending on how much importance we place on minority-class performance.

Recall the conclusion from observations on CIFAR-10-LT: (1) Initially, re-labeled head-class samples with a small proportion $\tau$ mostly invade tail-class feature space. (2) After re-labeling these critical samples, further label re-allocating relieves the head-class invasion of the medium-class feature space. (3) FedReLa prioritizes re-labeling samples that most severely invade tail-class regions. We observe similar results on the CIFAR-100-LT dataset, which are presented in Table~\ref{tab:sensitivity}. We anticipate that the optimal parameters will exhibit slight differences across datasets with varying posterior probability distributions and degrees of class overlap. When $\tau = 3$, FedReLa achieves maximum performance gain on CIFAR-100-LT, where the degree of class overlap is more severe. Although the parameter range 1--20\% consistently provides performance gain on both CIFAR-10 and CIFAR-100, with the principle of minimizing data-editing, \textbf{we recommend using slightly conservative relabeling strength (3\%--5\%).}

\begin{table}[H]
    \centering
    \caption{Sensitivity analysis on CIFAR-100-LT (IF=50, $\alpha=0.1$). Original denotes the FedETF baseline; other columns use FedReLa with top-$\tau$\% relabeling. The $\tau=5\%$ column matches the +FedReLa entry in Table~\ref{tab:lt}.}
    \label{tab:sensitivity}
    \begin{tabular}{lcccccccc}
        \toprule
        top-$\tau$\% & Original & 1 & 3 & 5 & 7 & 9 & 15 & 20 \\
        \midrule
        Overall     & 42.4 & 44.6 & \textbf{45.1} & 44.7 & 44.6 & 44.9 & 44.5 & 44.6 \\
        Many-shot   & 63.0 & 62.2 & 57.8 & 56.8 & 57.1 & 56.8 & 56.2 & 56.0 \\
        Medium-shot & 44.7 & 45.5 & 49.2 & 48.6 & 49.8 & 49.4 & 50.7 & 51.1 \\
        Few-shot    & 20.1 & 21.8 & 28.4 & 29.2 & 27.5 & 28.9 & 27.6 & 26.8 \\
        Relabeled   & 0    & 51   & 157  & 194  & 216  & 276  & 317  & 364  \\
        \bottomrule
    \end{tabular}
\end{table}

The threshold-tuning capability enables FedReLa to deliver customized class-wise enhancements, prioritizing tail-class gains while preserving overall performance. This strategic trade-off (suppressing overprivileged head classes to boost tails) is a unique advantage over static algorithm-level approaches \citep{Fed-ETF, FedLoGe}, as evidenced by the accuracy curves surpassing the baseline (dashed lines) in critical regions. Again, in practice, we can tune the trade-off through $t_{\rm re}^{(k)}$ by $\tau$ depending on how much importance we place on minority-class performance.

We report the sensitivity of the re-labeling start round $T_{\mathrm{relabel}}$ on long-tailed CIFAR-10/100 (IF=100) in Table~\ref{tab:trelabel_sensitivity}.


\begin{table}[H]
    \centering
    \caption{Sensitivity analysis of the relabeling start round $T_{\mathrm{relabel}}$ on long-tailed CIFAR-10/100 (IF=100). $T_{\mathrm{relabel}}$ is expressed as a fraction of the total training rounds $R$. Entries report Overall / Head / Med / Tail accuracy (\%).}
    \label{tab:trelabel_sensitivity}
    \resizebox{0.85\textwidth}{!}{
    \begin{tabular}{llcc}
        \toprule
        Dataset & $T_{\mathrm{relabel}}$ &
        $\alpha=1$ (O / H / M / T) &
        $\alpha=10$ (O / H / M / T) \\
        \midrule
        \multirow{4}{*}{CIFAR-10}
        & 40\%$R$ & 75.55 / 90.84 / 73.72 / 65.46 & 75.31 / 91.72 / 75.78 / 62.65 \\
        & 60\%$R$ & 76.52 / 91.04 / 75.02 / 66.76 & 76.80 / 92.95 / 76.77 / 64.71 \\
        & 80\%$R$ & \textbf{76.65} / 91.43 / 75.72 / \textbf{66.26} & \textbf{76.95} / 92.90 / 77.21 / \textbf{64.79} \\
        & FedETF baseline & 74.23 / 93.17 / 74.83 / 59.57 & 74.80 / 93.67 / 76.03 / 59.72 \\
        \midrule
        \multirow{4}{*}{CIFAR-100}
        & 40\%$R$ & 40.00 / 67.62 / 53.70 / 16.63 & 40.80 / 65.98 / 56.80 / 15.61 \\
        & 60\%$R$ & 43.53 / 64.43 / 49.83 / 28.06 & 44.01 / 64.32 / 54.40 / 25.11 \\
        & 80\%$R$ & \textbf{43.84} / 65.13 / 50.12 / \textbf{28.12} & \textbf{44.54} / 67.53 / \textbf{54.01} / \textbf{24.52} \\
        & FedETF baseline & 42.60 / 72.41 / 50.11 / 21.19 & 42.81 / 71.73 / 51.44 / 19.82 \\
        \bottomrule
    \end{tabular}
    }
\end{table}

Key observations:
\begin{itemize}
    \item Performance gain over the baseline is relatively insensitive to $T_{\mathrm{relabel}}$ once it reaches 60\% of the total training rounds, across different datasets and heterogeneity levels.
    \item Results show clear gains in both overall accuracy and minority (tail-class) accuracy compared to the baseline across all tested scenarios.
    \item We recommend setting $T_{\mathrm{relabel}} \approx 80\%R$ for robust performance across various datasets and heterogeneity levels.
\end{itemize}

\section{Ablation study}\label{sec:ablation}

\paragraph{Ablation study on the importance of Z-score standardization}
Z-score standardization is critical for enabling FedReLa to utilize the underestimated posterior probabilities output by biased models. To validate its necessity, we conducted ablation experiments on CIFAR-10-LT ($\alpha=0.1$, $\mathrm{IF}=50$) without standardization, and directly using posterior probabilities as flip probabilities.
\begin{table}[H]
\centering
\begin{tabular}{lcccc}
\textbf{Method} & Overall & \textbf{Many-shot} & \textbf{Medium-shot} & \textbf{Few-shot} \\
\hline
FedLOGE & 57.5 & 83.0 & 61.1 & 19.8 \\
\hline
+FedReLa & \textbf{70.0} & 76.0 & \textbf{72.7} & \textbf{59.4} \\
\hline
+FedReLa w/o Z-score & 59.7 & 82.1 & 72.3 & 17.4 \\
\hline
\end{tabular}
\caption{Performance of FedReLa with/without Z-score Standardization.}
\label{tab:zscore}
\end{table}

Without z-score standardization, the Few-shot performance fails to show improvement. This is attributed to the fact that the posterior probabilities are underestimated by the biased global model for tail classes and are typically extremely small. Directly utilizing them as flipping probabilities hinders the effective conversion of these samples into global minority classes. Meanwhile, the Medium-shot performance exhibits improvement as these classes possess more samples than tail classes, resulting in the model underestimating their posterior probabilities to a lesser extent. Thus, head-class samples with similar features are preferentially flipped to the medium class, rather than to the tail classes with tiny posterior probabilities.

Applying z-score standardization to the underestimated posterior probabilities enables a balanced label re-allocating behavior, which achieves a better balanced trade-off among the performance of Head, Medium, and Tail classes. Ultimately, this contributes to the superior Overall accuracy.

\paragraph{Ablation study on data-heterogeneity.}
To evaluate FedReLa's performance under higher imbalance ratios across varying degrees of data heterogeneity, we increased the imbalance ratio (IR) to 20 and the number of clients to $K=100$ on the Fashion-MNIST dataset with 3 minority classes.

Results in Table \ref{tab:ablation} depict consistent performance improvements by FedReLa across different levels of data heterogeneity on the Fashion-MNIST dataset for each algorithm-level method. This highlights the robustness of FedReLa in mitigating the impact of data heterogeneity through enhancements to both data and classifiers.

Improved percentage shows that the improvement achieved by FedReLa increases with higher data heterogeneity, indicating that FedReLa-boosted models exhibit significantly improved robustness to heterogeneous data distributions compared to baseline methods. 

As $\alpha$ decreases (i.e., heterogeneity increases), FedReLa demonstrates progressively greater improvements in both minority-class accuracy (+8.13\% to +35.40\%) and overall accuracy (+1.83\% to +10.82\%), with the most significant gains observed under extreme non-IID scenarios ($\alpha=0.1$). 
While baseline methods exhibit varied sensitivity to heterogeneity, CLIMB shows inherent robustness but limited enhancement headroom, and MOON suffers significant performance drops at $\alpha=0.1$. Yet FedReLa consistently mitigates these limitations through adaptive calibration, offering consistent enhancement. 
Notably, FedReLa reduces minority-class accuracy disparities by 23-37\% across $\alpha\leq 1$ while maintaining global model stability, particularly excelling in balancing the accuracy trade-off between dominant and rare classes. 
These results position FedReLa as a versatile solution for real-world federated learning deployments, offering three key advantages: 1) enhanced robustness to severe data heterogeneity without requiring client-specific tuning, 2) compatibility with existing aggregation frameworks, and 3) simultaneous optimization of both class-balanced and global model performance in non-IID environments.
\newpage


\begin{table}[H]
    \centering
    \caption{Sensitivity to relabel refresh period on CIFAR-10/100 (overall test accuracy in \%).
    Columns ``FedETF'' and ``+FedReLa (one-shot)'' report the baseline and a single end-of-training relabeling pass for reference.
    Columns ``Period~1''--``Period~5'' refresh the relabel mapping every 1--5 communication rounds, respectively, and report the best overall accuracy achieved during training.}
    \label{tab:period_sensitivity}
    \resizebox{0.85\textwidth}{!}{
    \begin{tabular}{l|ccccccc}
        \toprule
        & FedETF & +FedReLa (one-shot) & Period~1 & Period~2 & Period~3 & Period~4 & Period~5 \\
        \midrule
        CIFAR-10, IF=100, $\alpha=0.3$ & 42.88 & 46.00 & 48.53 & 48.78 & \textbf{49.17} & 48.56 & 48.92 \\
        CIFAR-100, IF=50, $\alpha=0.1$ & 42.35 & 44.71 & \textbf{45.04} & 44.70 & 44.48 & 44.67 & 44.57 \\
        \bottomrule
    \end{tabular}
    }
\end{table}

\paragraph{One-shot vs. Periodic}
We compare a single end-of-training relabeling pass (one-shot) with periodically refreshed mappings (Period~1--5), where the relabeling rule is updated every few communication rounds.
Periodic refresh incurs a modest additional local computation cost, but Table~\ref{tab:period_sensitivity} shows that it can pay off: on CIFAR-10, periodic strategies consistently yield stronger overall gains than one-shot alone, whereas on CIFAR-100 the extra benefit is more limited.
One-shot relabeling already improves over the FedETF baseline on both datasets; when the compute budget allows, adopting a periodic refresh schedule is a practical way to push performance further, especially in the CIFAR-10 regime.

\begin{table}[H]
\centering
\resizebox{0.8\textwidth}{!}{
\begin{tabular}{|c|cccc|}
\hline
        & \multicolumn{2}{c|}{Original (+FedReLa) Performance}           & \multicolumn{2}{c|}{Improved Percentage (\%)} \\ \hline
        & \multicolumn{1}{c|}{Minority Accuracy} & \multicolumn{1}{c|}{Overall Accuracy} & \multicolumn{1}{c|}{Minority Accuracy} & Overall Accuracy \\ \hline
        & \multicolumn{4}{c|}{$\alpha=10$} \\ \hline
FedAvg  & \multicolumn{1}{c|}{51.70 (81.83)}    & \multicolumn{1}{c|}{79.57 (84.56)}   & \multicolumn{1}{c|}{+30.13}    &     \multicolumn{1}{c|}{+4.99}             \\ \cline{1-1}
FedProx & \multicolumn{1}{c|}{51.57 (81.83)}    & \multicolumn{1}{c|}{79.35 (84.40) }   & \multicolumn{1}{c|}{+30.26}    &    \multicolumn{1}{c|}{+5.05}\\ \cline{1-1}
FedNova & \multicolumn{1}{c|}{52.00 (82.27)}    & \multicolumn{1}{c|}{79.35 (84.61)}   & \multicolumn{1}{c|}{+30.27}    &      \multicolumn{1}{c|}{+5.26}            \\ \cline{1-1}
MOON    & \multicolumn{1}{c|}{45.77 (81.17)}    & \multicolumn{1}{c|}{78.35 (85.60)}   & \multicolumn{1}{c|}{+35.40}    &      \multicolumn{1}{c|}{+7.25}            \\ \cline{1-1}
CLIMB   & \multicolumn{1}{c|}{55.80 (69.57)}    & \multicolumn{1}{c|}{82.39 (86.14)}   & \multicolumn{1}{c|}{+13.77}    &      \multicolumn{1}{c|}{+3.75}            \\ \hline & \multicolumn{4}{c|}{$\alpha=5$} \\ \hline
FedAvg  & \multicolumn{1}{c|}{45.27 (78.53)}    & \multicolumn{1}{c|}{77.80 (84.42)}   & \multicolumn{1}{c|}{+33.26}    &     \multicolumn{1}{c|}{+6.62}             \\ \cline{1-1}
FedProx & \multicolumn{1}{c|}{45.57 (78.03)}    & \multicolumn{1}{c|}{77.90 (84.27)}   & \multicolumn{1}{c|}{+32.46}    &    \multicolumn{1}{c|}{+6.37}\\ \cline{1-1}
FedNova & \multicolumn{1}{c|}{46.10 (78.60)}    & \multicolumn{1}{c|}{78.18 (84.30)}   & \multicolumn{1}{c|}{+32.50}    &      \multicolumn{1}{c|}{+6.12}            \\ \cline{1-1}
MOON    & \multicolumn{1}{c|}{45.57 (79.73)}    & \multicolumn{1}{c|}{78.21 (85.51)}   & \multicolumn{1}{c|}{+34.16}    &      \multicolumn{1}{c|}{+7.30}            \\ \cline{1-1}
CLIMB   & \multicolumn{1}{c|}{56.10 (69.53)}    & \multicolumn{1}{c|}{82.67 (86.35)}   & \multicolumn{1}{c|}{+13.43}    &      \multicolumn{1}{c|}{+3.68}            \\ \hline
& \multicolumn{4}{c|}{$\alpha=1$} \\ \hline
FedAvg  & \multicolumn{1}{c|}{50.67 (75.57)}    & \multicolumn{1}{c|}{79.06 (84.12)}   & \multicolumn{1}{c|}{+24.90}    &     \multicolumn{1}{c|}{+5.06}             \\ \cline{1-1}
FedProx & \multicolumn{1}{c|}{50.10 (74.60)}    & \multicolumn{1}{c|}{78.95 (84.25)}   & \multicolumn{1}{c|}{+24.50}    &    \multicolumn{1}{c|}{+5.30}\\ \cline{1-1}
FedNova & \multicolumn{1}{c|}{50.10 (75.93)}    & \multicolumn{1}{c|}{78.91 (84.47)}   & \multicolumn{1}{c|}{+25.83}    &      \multicolumn{1}{c|}{+5.56}            \\ \cline{1-1}
MOON    & \multicolumn{1}{c|}{44.87 (74.97)}    & \multicolumn{1}{c|}{77.94 (84.95)}   & \multicolumn{1}{c|}{+30.10}    &      \multicolumn{1}{c|}{+7.01}            \\ \cline{1-1}
CLIMB   & \multicolumn{1}{c|}{61.17(69.3)}    & \multicolumn{1}{c|}{83.78(85.93)}   & \multicolumn{1}{c|}{+8.13}    &      \multicolumn{1}{c|}{+2.15}            \\ \hline
        & \multicolumn{4}{c|}{$\alpha=0.3$}\\ \hline
FedAvg  & \multicolumn{1}{c|}{50.50 (75.70)}    & \multicolumn{1}{c|}{78.46 (83.84)}   & \multicolumn{1}{c|}{+25.20}    &     \multicolumn{1}{c|}{+5.38}             \\ \cline{1-1}
FedProx & \multicolumn{1}{c|}{50.00 (74.90)}    & \multicolumn{1}{c|}{78.48 (83.63) }   & \multicolumn{1}{c|}{+24.90}    &    \multicolumn{1}{c|}{+5.15}\\ \cline{1-1}
FedNova & \multicolumn{1}{c|}{55.03 (77.90)}    & \multicolumn{1}{c|}{76.67 (84.46)}   & \multicolumn{1}{c|}{+22.87}    &      \multicolumn{1}{c|}{+7.79}            \\ \cline{1-1}
MOON    & \multicolumn{1}{c|}{44.43(74.77) }    & \multicolumn{1}{c|}{77.12(84.47) }   & \multicolumn{1}{c|}{+30.34}    &      \multicolumn{1}{c|}{+7.35}            \\ \cline{1-1}
CLIMB   & \multicolumn{1}{c|}{53.43 (64.70)}    & \multicolumn{1}{c|}{81.42 (84.56)}   & \multicolumn{1}{c|}{+11.27}    &      \multicolumn{1}{c|}{+3.14}            \\ \hline
        & \multicolumn{4}{c|}{$\alpha=0.1$}\\ \hline
FedAvg  & \multicolumn{1}{c|}{33.83 (67.60)}    & \multicolumn{1}{c|}{68.43 (79.25)}   & \multicolumn{1}{c|}{+33.77}    &    \multicolumn{1}{c|}{+10.82} \\ \cline{1-1}
FedProx & \multicolumn{1}{c|}{34.40 (68.27)}    & \multicolumn{1}{c|}{69.37 (79.35)}   & \multicolumn{1}{c|}{+33.87}    &      \multicolumn{1}{c|}{+9.98}             \\ \cline{1-1}
FedNova & \multicolumn{1}{c|}{70.43 (82.83)}    & \multicolumn{1}{c|}{74.25 (82.10)}   & \multicolumn{1}{c|}{+12.40}    &     \multicolumn{1}{c|}{+7.85}\\ \cline{1-1}
MOON    & \multicolumn{1}{c|}{22.59 (45.82)}    & \multicolumn{1}{c|}{67.44 (77.88)}   & \multicolumn{1}{c|}{+23.23}    &    \multicolumn{1}{c|}{+10.44} \\ \cline{1-1}
CLIMB   & \multicolumn{1}{c|}{56.97 (65.20)}    & \multicolumn{1}{c|}{81.78 (83.61)}   & \multicolumn{1}{c|}{+8.23}    &      \multicolumn{1}{c|}{+1.83}           \\ \hline
\end{tabular}
}
\caption{Ablation study on $\alpha$. The overall accuracy and average accuracy of minority classes (in \%) on step-wise Fashion-MNIST with 3 minority classes (30\%) for $\text{IR} = 20$ with 100 clients. The results in brackets show the FedReLa enhanced performance.}\label{tab:ablation}
\end{table}

\end{document}